\documentclass[11pt,english]{article}
\usepackage[LGR,T1]{fontenc}
\usepackage[latin9]{inputenc}
\usepackage{geometry}
\usepackage{color}
\definecolor{page_backgroundcolor}{rgb}{1, 1, 1}
\pagecolor{page_backgroundcolor}
\usepackage{mathrsfs}
\usepackage{enumitem}
\usepackage{amsmath}
\usepackage{amsthm}
\usepackage{amssymb}
\usepackage{setspace}
\usepackage{esint}
\makeatletter

\DeclareRobustCommand{\greektext}{%
  \fontencoding{LGR}\selectfont\def\encodingdefault{LGR}}
\DeclareRobustCommand{\textgreek}[1]{\leavevmode{\greektext #1}}
\ProvideTextCommand{\~}{LGR}[1]{\char126#1}

\numberwithin{equation}{section}
\numberwithin{figure}{section}
\theoremstyle{remark}
\newtheorem*{rem*}{\protect\remarkname}
\newenvironment{lyxcode}
	{\par\begin{list}{}{
		\setlength{\rightmargin}{\leftmargin}
		\setlength{\listparindent}{0pt}% needed for AMS classes
		\raggedright
		\setlength{\itemsep}{0pt}
		\setlength{\parsep}{0pt}
		\normalfont\ttfamily}%
	 \item[]}
	{\end{list}}
\theoremstyle{plain}
\newtheorem{conjecture}{\protect\conjecturename}
\theoremstyle{plain}
\newtheorem{prop}{\protect\propositionname}
\theoremstyle{definition}
\newtheorem{defn}{\protect\definitionname}
\newlist{casenv}{enumerate}{4}
\setlist[casenv]{leftmargin=*,align=left,widest={iiii}}
\setlist[casenv,1]{label={{\itshape\ \casename} \arabic*.},ref=\arabic*}
\setlist[casenv,2]{label={{\itshape\ \casename} \roman*.},ref=\roman*}
\setlist[casenv,3]{label={{\itshape\ \casename\ \alph*.}},ref=\alph*}
\setlist[casenv,4]{label={{\itshape\ \casename} \arabic*.},ref=\arabic*}
\theoremstyle{plain}
\newtheorem{thm}{\protect\theoremname}
\theoremstyle{plain}
\newtheorem{cor}{\protect\corollaryname}
\theoremstyle{plain}
\newtheorem{lem}{\protect\lemmaname}

\makeatother

\usepackage{babel}
\providecommand{\casename}{Case}
\providecommand{\conjecturename}{Conjecture}
\providecommand{\corollaryname}{Corollary}
\providecommand{\definitionname}{Definition}
\providecommand{\lemmaname}{Lemma}
\providecommand{\propositionname}{Proposition}
\providecommand{\remarkname}{Remark}
\providecommand{\theoremname}{Theorem}

\begin{document}
\title{The use of knowledge in open-ended systems}
\author{Abigail Devereaux\thanks{Wichita State University}, Roger Koppl\thanks{Syracuse University}}
\maketitle
\begin{abstract}
Economists model knowledge use and acquisition as a cause-and-effect
calculus associating observations made by a decision-maker about their
world with possible underlying causes. Knowledge models are well-established
for static contexts, but not for contexts of innovative and unbounded
change. We develop a representation of knowledge use and acquisition
in open-ended evolutionary systems and demonstrate its primary results,
including that observers embedded in open-ended evolutionary systems
can agree to disagree and that their ability to theorize about their
systems is fundamentally local and constrained to their frame of reference
(what we call ``frame relativity''). The results of our framework
formalize local knowledge use, the ``many-selves'' interpretation
of reasoning through time, and motivate the emergence of nonlogical
modes of reasoning like institutional and aesthetic codes.
\end{abstract}

\section{Introduction}

The central problem in epistemology is the discovery of a complete
and correct set of statements about the system in which some observer
is embedded, and is well-established for closed systems (Hintikka
1962; Aumann 1999a, 1999b; Samet 1990). Using modal logic, theorists
define common knowledge situations in game theoretic contexts and
in other closed systems in which the universe of possible states has
been pre-stated. However, a thorn in the side of decision theory has
long been the salience in social systems of truly novel possibilities.
There are more things in heaven and earth than are dreamt of in a
closed system.

Truly novel possibilities are generated by open-ended systems, rendering
questionable the applicability of methods suited to closed systems.
Open-ended systems that generate novel possibilities require embedded
observers to revise and replace theories of the system confidently
employed in previous periods. The epistemological characterization
of observer knowledge embedded within an open-ended evolutionary system
under theory selection remains an open problem without a tractable
solution. 

Open-ended evolutionary (OEE) processes are processes that continually
increase in complexity, generate novel change, and are both unbounded
and innovative (Banzhaf et al 2016; Adams et al 2017; Corominas-Murtra
et al 2018). Technological innovation is an open-ended process borne
of open-ended entrepreneurial action (Arthur 2009; Koppl et al 2023).
At a social level, complex technologies like language can evolve in
an open-ended fashion (Chaitin 2017b). New technological possibilities
and ways to communicate can generate opportunities to improve one's
lot along some metric like income or happiness. Exploiting these new
opportunities requires an individual to acquire knowledge about them.
The process of acquiring and using knowledge about new possibilities
generated by OEE processes is itself open-ended, and therefore cannot
be described by (closed-ended) neoclassical economic theories (Giménez
Roche 2016). Innovation is deeply salient to the lives of decision-making
individuals---the rapid ascendance of large-language models merely
the most recent example of two centuries punctuated by frequent society-shifting
innovations. It is high time to develop a knowledge framework encompassing
of open-endedness. 

Knowledge frameworks allow us to describe decision-making using the
formal language of epistemic logic. Models in epistemic logic describe
how observers embedded within systems discover and classify the causes
of observed system characteristics or effects. The epistemics of rational
decision-making in economics require several strong axiomatic assumptions,
namely the recursive enumeration of a forever growing ``grand''
state space ``whose elements describe anything that can possibly
be of interest'' (Gilboa \& Marinacci 2016: 11). Individuals comprehend
causal patterns, generate likelihood distributions, and correct errors
of belief. To model this process of knowledge use and acquisition,
economists often rely on an analytical framework exemplified by Aumann\textquoteright s
famous paper \textquotedblleft Agreeing to Disagree\textquotedblright{}
(Aumann 1976) and his later paper ``Epistemology I'' (Aumann 1999a).
Aumann\textquoteright s paper has been recognized to be congruent
with the earlier work of Hintikka (1962) and Kripke (1963). Artemov
(2022) speaks of \textquotedblleft Kripke/Aumann models.\textquotedblright{} 

Economists define rational decision-making in terms of closed-ended
analysis allowing for the ranking of all alternatives. In this type
of system individuals apply logic or probabilistic\footnote{The partition conception of information and the subset conception
of probability have dual interpretations (Kung et al 2009; Ellerman
2022). So it is possible for us to consider non-probabilistic partitionable
possibles, knowing there is a direct analogy to probabilistic partitionable
possibles and allowing us to generalize to probabilistic and Bayesian
reasoning without requiring a separate argument. Possibles are also
analogizable to quantum states, though we are more interested in propositional
possibles. But, indeed, conceiving the possible in both probabilistic
and non-probabilistic partitionable possibles boils down to a combinatorial
exercise.} inference over a set of observations to fully partition all possible
states of the universe. Individuals then formulate mappings between
sets of observations and possible worlds, where correct cause-and-effect
classifications are referred to as knowledge (see the following for
accounts of probabilistic state-space inference: Anscombe \& Aumann
1963; Dempster 1967; Gilboa \& Schmeidler 1989, 1993; Gilboa et al
2008). As per Chaitin (2005: 6), ``Understanding is compression!''
Knowledge \emph{qua} optimization boils down to the calculation of
a topological fixed point, necessitating strong conditions on what
individuals believe is possible and what is really possible. Mathematically
consistent, coherent and complete classifications are essential to
the existence (though not to the computable discovery\footnote{Cf. the proof of the non-computability of excess demands by Kenneth
Arrow's protégé, Alain Lewis (1985).}) of a fixed point. But innovative change shatters consistency in
unpredictable (and unapproximable) ways. Decision-making under uncertainty
in economics is organized under the banner of rational choice and,
as such, precludes explanation of innovative change and open-endedness
(Gilboa 2023). Aumann (1999b) has shown that the Bayesian subjective
belief calculus, perhaps the most popular way of modeling decision-making
under uncertainty in economics, is logically equivalent to traditional
epistemology and thus can only represent Knightian risk. Our framework,
instead, models open-ended evolution and Knightian uncertainty. 

The Kripke/Aumann models used in economic theory reflect the closed-ended
equilibrium methods with which they were developed in tandem. The
traditional epistemological formalism derives from the effort to make
commensurate the semantic (modal logic) construction of epistemology
in game theoretic settings with the syntactic (propositional logic)
construction (Artemov 2022). Four key results of these models are
as follows: 
\begin{enumerate}
\item (CE1) An economic agent can codify all concepts relevant to decision-making
in a complete and fully partitionable state space; 
\item (CE2) The state space is the only and final universe and that different
worlds correspond to different states; 
\item (CE3) All economic agents have access to the complete and correct
theory of the state-space universe; and 
\item (CE4) Common knowledge is possible. 
\end{enumerate}
Results CE1-4 do not generally hold in an OEE framework for modeling
epistemic logic, as we shall demonstrate. In OEE systems individuals
must be able to revise and replace theories employed in previous periods
as emergent novelty generates incommensurability. Knowledge is understood
from the perspective of observers embedded inside of OEE systems,
rather than from the perspective of the theorist sitting outside the
system. Furthermore, individuals must have a way to resolve issues
like undecidable disjunctions without access to an overarching theory
of the system. 

In this paper, we develop an open-ended evolutionary knowledge framework
encompassing of innovative change. Though based on similar formal
principles, our framework moves beyond the traditional epistemology
of rational choice and game theory. We propose a model of epistemic
logic that allows for radical change in an individual\textquoteright s
interpretative framework.Five key results of our OEE knowledge framework
are as follows: 
\begin{enumerate}
\item (OEE1) The set of believed-to-be-true propositions and the set of
true propositions for any observer embedded in an OEE system can never
be entirely coincident; 
\item (OEE2) Common knowledge is in general impossible in an OEE system; 
\item (OEE3) It is possible for individuals in OEE systems to agree to disagree;
\item (OEE4) Nonlogical (random, heuristic, aesthetic) searches in OEE systems
can be as knowledge generative as logical search mechanisms; and
\item (OEE5) Knowledge in open-ended evolutionary systems is non-ergodic. 
\end{enumerate}
To derive results OEE1-5 we demonstrate that knowledge in OEE systems
is fundamentally local and fragmented between individuals. Furthermore,
we introduce the concept of \emph{frame relativity}, a formalization
of the more familiar statement that what an embedded observer can
know and the full description of what is possible in an OEE system
are never entirely coincident. True possibilities lie always outside
of one's knowledge and cannot enter within any estimation. As Keynes
(1937, p. 114) said, \textquotedbl We simply do not know.\textquotedbl{}
The OEE knowledge framework is a fertile ground for discovering ways
in which economic agents choose that differ from what is possible
within traditional epistemology. These ways of choosing include, for
example, acting on aesthetic considerations.

In Section \ref{sec:Preliminaries} we bring the reader up to speed
on open-ended evolution and the theory of knowledge use and acquisition
in economics. In Section \ref{sec:Math-section} we overview our framework
for open-ended evolutionary knowledge acquisition. In Section \ref{sec:Interpretation-and-implications}
we demonstrate our primary results, focusing on Results OEE1-5. We
then conclude. We provide the formal exposition of the theory in the
Appendix including theorems and proofs referred to in the text.

\section{\label{sec:Preliminaries}Preliminaries}

\subsection{\label{subsec:Traditional-epistemology:-formal}Traditional epistemology:
formalism}

Traditional epistemology starts with a language $L$ including a set
of logical and nonlogical constants with which individuals form statements
about the world. Define a universe $\Omega$ as the set of all possible
states $\omega$. Define a state $\omega$ as a set of state-specific
true statements (also called formulas or sentences) $\xi$. States
marry context and facts. The universe, therefore, is a complete and
contextualized set of facts. For instance, $\xi=$''the sky is cloudy''
may hold in state $\omega_{1}$ but not in state $\omega_{2}$. Facts
associated with states can be characterized as simple sets. 

By construing the $\xi$ as statements and the $\omega$ as sets of
statements we are adopting the analytical perspective pioneered by
Samet (1990). Aumann (1999a, p. 266) has adopted Samet\textquoteright s
formalism, describing it as \textquotedblleft the simple but ingenious
and fundamental idea of formally characterizing a state of the world
by the sentences that hold there.\textquotedblright{} Samet\textquoteright s
formalism begins with a countable set $\Phi$ of propositions and
a countable set $I$ of individuals. For each such individual, \ensuremath{i},
there is a knowledge function $K_{i}:\Phi\rightarrow\Phi$. If $\phi\in\Phi$,
then $K_{i}\phi$ means \textquotedblleft $i$ knows $\phi$.\textquotedblright{}
Samet\textquoteright s crucial move was then to define the function
$\Sigma=\{0,1\}^{\Phi}$. Each element of $\Sigma$ can be thought
of as an assignment of truth values to the propositions: 1 for \textquoteleft true\textquoteright{}
and 0 for \textquoteleft false.\textquoteright{} This move allows
him to define a \textquotedblleft state of the world'' $\omega$
as any element of $\Sigma$ that satisfies the condition $\omega(\phi)+\omega(\lnot\phi)=1$,
that is, the observed states of the world cannot be true and false
at the same time. This subset of $\Sigma$ is called the event space,
$\Omega_{0}$. By this \textquotedblleft simple but ingenious\textquotedblright{}
method, Samet moves from a set of propositions about the world to
a countably infinite event space. 

Individuals are endowed with partitions of states of the universe
and associate the partitions with events to form a cause-and-effect
belief calculus. The partitions $\mathscr{I}$ of each individual
are fixed and are common knowledge. The partition function of individual
$i$ is defined by a knowledge function $\kappa_{i}$ defined over
$\Omega$ such that $\mathscr{I}(\omega)=\{\omega'\in\Omega:\kappa_{i}(\omega)=\kappa_{i}(\omega')\}$.For
any state $\omega$, $\kappa_{i}(\omega)$ is a set of statements
with defined truth values, i.e., all statements prepended with the
function $k_{i}$. Partitions form tractably closed-and-bounded topological
spaces, similar to probabilistic state-space inference. Individuals
cannot make a choice undictated by the closed and complete logic of
the topology (Aumann \& Brandenberger 1995). 

Aumann (1999a,b) showed the equivalence between the semantic construction
of knowledge and syntactic epistemic logic (SEL). We adopt the SEL
approach to formalizing knowledge in this paper. In this interpretation,
suppose there is a theory of everything $T_{\Omega}$ based on a language
$L_{\Omega}$ from which all correct statements about the world can
be derived, for all time. To reason coherently, an individual $i$
at time $t$ must reason from a theory $T_{i,t}$ that conforms to
$T_{\Omega}$ in specific ways (as discussed in more length in the
Appendix). Moreover, if \textbf{$\mathfrak{L}$ }is a list of sentences
$\xi$ in a language $L$, then \textbf{$\mathfrak{L}$} is ``epistemically
closed'' if $f\in\mathfrak{L}\Rightarrow k_{i}f\in\mathfrak{L}$.
And, finally,  $i$ knows the knowledge partition of $j$, and $j$
knows that $i$ knows, and so on for all $i$ and $j$.

The SEL construction of traditional epistemology is straightforward
compared to the semantic construction, but requires making metamathematical
decisions about the character of the possibility space. For example,
SEL models build in the ability of individuals to discern between
states of reality (decidability) and an orderly listability (recursive
enumerability) of possibilities, but without an explanation of where
these characteristics come from. Probabilistic variations of traditional
epistemology which attempt to encompass uncertainty are, if tractable,
invariably equivalent in their results, implications and reasoning
to non-probabilistic traditional epistemology, as is Bayesian ``subjective''
knowledge theory (Aumann 1999b) and variations like the Anscombe-Aumann
(1963) framework. 

The assumptions about the state space and the universe in traditional
epistemology are strong and several and equivalent to the axioms of
the modal logic system $S5_{n}$ (cf. the Appendix). Weakening any
assumption degrades the predictive power of the model. For example,
completely partitionable spaces are necessary to reliably predict
common knowledge, and common knowledge is necessary to reliably predict
dynamical and strategic features like equilibria and trigger strategies.
Common knowledge of partitions is asserted and not proved, and is
axiomatic (Aumann 1999a: 277).

The heavy lifting of formulating statements, observing reality, and
deciding between a statement and its negation comes in the presumptive
step of partitioning the possibility space. Since CE knowledge is
defined as correct belief, partitions that encode knowledge must have
somehow verified the consistency or truth of statements relative to
a (fixed) theoretical interpretation of an existing body of knowledge.
Partition functions skip over the process of formulating statements
and hypotheses, making observations, and coming to conclusions about
statements based on those observations---the process of knowledge
use and acquisition itself. 

\subsection{\label{sec:A-primer-on-open-ended-evolution}Open-ended evolution
in economics}

While OEE is discussed at least as far back as Bergson (2014 {[}1889{]},
1911) in his concepts of ``qualitative multiplicity'' (a whole not
reducible to its parts), ``duration'' (a process not reducible to
a trajectory), and ``creative evolution'' (evolution that self-generates
novelty), the literature formalizing OEE processes is sparse, with
most entries from computer science theorists studying complex evolutionary
behavior in the simple ``automata'' programs discovered by von Neumann
(1948 {[}1951{]}, 1949 {[}1966{]}).

Aspects of open-ended evolutionary dynamics have received recent treatments
in artificial life (Taylor 1999, 2015, 2019), biology (Bedau \& Packard
1998; Ruiz-Mirazo et al 2008; Corominas-Murtra et al 2018), theoretical
chemistry (Duim \& Otto 2017), computer science (Wolfram 2002; Huneman
2012; Hernández-Orozco et al 2018), physics (Adams et al 2017) and
pure theory (Taylor et al 2016; Banzhaf et al 2016). Treatments of
open-endedness in economics are sparser and tend to center on negative
results which reject aspects of the applicability of neoclassical
theorizing (Giménez Roche 2016).

These various literatures are not entirely consistent in how they
define OEE. OEE can be characterized by: a continual, endogenous generation
of novelty (Banzhaf et al 2016; Adams et al 2017; Hernández-Orozco
et al 2018), an increase in system complexity over time (Corominas-Murtra
et al 2018), self-referential and reflexive processes (Wolfram 2002;
Giménez Roche 2016; Adams et al 2017), the emergence of structures
or possibilities (Giménez Roche 2016; Banzhaf et al 2016; Adams et
al 2017; Corominas-Murtra et al 2018), and an unlistability of elements
generated by the OEE process (Kauffman \& Roli 2021; Bedau et al 1998;
Taylor 1999; Ruiz-Mirazo et al 2008). 

In Banzhaf et al (2016) in particular, the authors center their analysis
on how continuous novelty\footnote{Novelty in economic models can mean anything from individuals encountering
an unexpected variation of a known possibility (a known unknown, or
Knightian risk) (cf. Loreto et al 2016) to an entirely novel and unprestatable
new possibility with ramifications across their choice space (an unknown
unknown, or Knightian uncertainty). Open-ended evolution, in contrast
with evolutionary fitness landscapes, is generative of unknown unknowns
(Hernández-Quiroz, \& Zenil 2018). } generates most of the features associated with OEE systems including
unbounded innovation, increasing system complexity, and emergence.
Their framework is theory-and-language based and describes several
types of novelty, where the epistemically simplest type of novelty
varies parameters of existing representative models, the intermediate
type of novelty requires model alterations like the addition of variables,
and the most epistemically challenging type of novelty requires alterations
of the theory underlying representative models\footnote{Type 0 novelty, or \emph{variation} within-the-model, ``explores
a pre-defined (modeled) state space, producing new values of existing
variables'' (Banzhaf et al 2016: 141). An example of a Type 0 variation
is when an individual changes the value of the risk aversion parameter
in their utility function to become more risk-averse. Type 1 novelty,
or \emph{innovation} that changes the model, ``adds a new type or
relationship that conforms to the meta-model, or possibly eliminates
an existing one'' (ibid). A Type 1 innovation could change the form
of an individual's utility function from Cobb-Douglas to Leontief.
Type 2 novelty, or \emph{emergence,} changes the theory itself. Emergence
is a phenomenon that cannot be explained within an existing theory,
like how everywhere-efficient theories of rational choice cannot explain
the orderly movement of skaters around a roller rink.}. Only the most epistemically challenging type of novelty is capable
of exhibiting emergence, and is therefore the type of novelty generated
by OEE systems.

Emergent system states not reducible to simple combinations of their
parts abound in social and biological systems (Kauffman 1993; Silberstein
2002; Rosas et al 2024). The incommensurability of old and new theories
caused by emergence drives scientific revolutions, the new theory
displacing the old until it is itself displaced, and so on (Kuhn 1996;
Feyerabend 1993; Nickles 2008). Recent disruptions in physics include
modeling the expansion of quantum possibility spaces using isometry
and Feynman path integrals (Cotler \& Strominger 2022), and the geometrical
volume interpretation of particle collision (Arkani-Hamed \& Trnka
2014). Recent disruptions in economic theory include the move towards
experimental techniques (Smith 2003), the empirical-Bayesian ``credibility
revolution'' (Angrist \& Pischke 2010), and the adoption of methods
from complexity theory (Arthur 2015; Helbing \& Kirman 2013; Haldane
\& May 2011). Theory alteration transcends the formal production of
science, an individual's perception of reality and in turn, their
choice behavior. 

\section{\label{sec:Math-section}The OEE framework: motivation and basic
theorems}

Grappling with theory-altering novelty motivates the construction
of our OEE knowledge framework, which we present in this section.
Modeling knowledge the usual way implies Results CE1-4 listed above,
which are unable to cope with theory-altering novelty in OEE systems
as we shall demonstrate in detail in Section \ref{sec:Interpretation-and-implications}.
It is clear that constructing the knowledge process as generative
of continual novelty and where knowledge of the grand state space
is fragmented or mostly hidden from each individual will weaken the
axioms of any candidate model of epistemic logic in OEE systems---beyond
the criticism that Aumann/Kripke models of epistemic logic are already
too strong with respect to modeling rational strategic behavior in
CE systems (Artemov 2022 successfully weakens the common knowledge
axiom). 

\subsection{The primary considerations of OEE systems relevant to knowledge}

As in CE models of knowledge as described in Section \ref{subsec:Traditional-epistemology:-formal},
formulating a cause-and-effect calculus to inform choices in OEE systems
starts with constructing a language $L$, where sentences $\xi$ are
constructed from a combination of logical constants (logical connectors
and ``grammar'') and nonlogical constants (aspects of reality like
descriptors and objects, what logicians call ``predicates'') (Aumann
1999a). The set of all sentences is called the syntax $\mathfrak{S}$.
Beyond these elements, the formal epistemology of OEE systems differs
from that of CE systems. Recall that theorizing in CE systems is from
the perspective of an external expert on behalf of other individuals,
who presumes there exists a theory-of-everything $T_{\Omega}$ of
the system $\Omega$ in the language $L$ for the universe in which
all realizable states are composed of true sentences and where sentences
must be coherent and complete such that there are no undecidable disjunctions
in the theory, i.e., no ontological truths that cannot be proved true
within the epistemological theory. 

The primary considerations of OEE systems are that: 
\begin{enumerate}
\item Theorizing is from the perspective of the embedded observer, who cannot
\emph{a priori} impose their perspective of reality on other agents.
\item By virtue of open-endedness, OEE knowledge theory necessitates an
individual-level process of continual theory revision.
\item Unbounded and innovative processes in OEE systems tend to grow the
number of possibilities in the system.
\item Formulating knowledge theory using propositional logic in CE models
leaves open questions of how a theory of the universe $T_{\Omega}$
is constructed to be decidable and how the theory used by each individual
$i$ at time $t$, $T_{i,t}$, relates to $T_{\Omega}$ and to individual
$j$'s theory of the universe $T_{j,t}$. 
\end{enumerate}
Consideration (1) \emph{localizes} knowledge acquisition and use to
the individual with respect to the individual's epistemic environment
and their conception of their system at a given time, $t$. Call the
individual's known-world $\Omega_{i,t}$. The individual understands
their known-world as a model $\text{\textbf{M}}_{i,t}$ defined within
a theory $T_{i,t}$. The individual's syntax for their known-world
is $\mathfrak{S}_{i,t}$. All these constructs are localized and do
not apply to the entire population $N$. 

Consideration (2) constructs knowledge acquisition and use as a \emph{novelty-generating
process}. For the purposes of exposition and without loss of generality,\footnote{We can always re-index periods in a way that makes preserves novelty
generation among a population of individuals, where periods are defined
as at least one individual in the system encountering theory-breaking
novelty. Bringing individual-specific temporal periods into the theory
right now would unnecessarily clutter our results without clarifying
very much.} we index $t$ such that the update rule requires theory revision
between periods. At time $t+1$, individual $i$ discovers novel possibilities
including first-order formulas that signify essential relationships
not explained or even listable within their existing theory. These
novel possibilities may add to or displace other possibilities after
theory revision. The set of all possible states of the universe is
therefore altered such that $\Omega_{i,t+1}\neq\Omega_{i,t}$. The
theory $T_{i,t}$ no longer sufficiently describes the logical relationships
between sentences of $\xi\in\Omega_{i,t+1}$ as it is now incomplete.
The new observed states may have logically extended $T_{i,t}$ or
contradicted old axioms, rendering $T_{i,t}$ inconsistent. In general,
 $i$ must select a new theory $T_{i,t+1}.$
\begin{rem*}
We essentially define time-periods $t$ in terms of the perceived
applicability by individual $i$ at time $t$ of a theory $T_{i,t}$.
Our definition differs from how $t$ is defined in discrete-dynamical
theories of economic growth, business cycles, and in agent-based computational
economics. Actual agent dynamics may contain entire worlds inside
each temporal cross-section. Within cross-sections, rational agents
may employ static and simplified theories like deterministic search
over apparent landscapes of possibilities. $T_{i,t}$ is a map with
which $i$ to can reasonably navigate a temporal cross-section. 
\end{rem*}
Consideration (3) constructs the novelty-generating processes of possibility
spaces in OEE systems as innovative and unbounded in growth. This
implies that $|\Omega_{i,t+1}|>|\Omega_{i,t}|$ in general. Logically,
in terms of individuals formulating languages of their known-worlds
in which they then theorize about their world, a continual increase
in possibilities means a continual addition to the set of all possible\emph{
}qualities, objects or characteristics (as opposed to merely altering
the value of a variable). In epistemology, this is called the set
$\mathcal{P}_{i,t}$ of all nonlogical constants $P_{k}$ (as opposed
to logical constants, which are mostly operators and relationships).
It is the continual growth $|\mathcal{P}_{i,t+1}|>|\mathcal{P}_{i,t}|$
that drives the need for theory revision in OEE systems. 

We can now define the \textbf{possible $\mathbf{\varPi}_{i,t}$} for
individual $i$ at time $t$ in an OEE system as a triplet of the
set of nonlogical constants $\mathcal{P}_{i,t}$, the theory $T_{i,t}$
of the cause-and-effect structure of the universe, and a decision-theoretic
model $\mathbf{M}_{i,t}$ consistent $T_{i,t}$ . We write $\mathbf{\varPi}_{i,t}$
as

\[
\mathbf{\varPi}_{i,t}=\langle\text{\textbf{M}}_{i,t},T_{i,t},\mathcal{P}_{i,t}\rangle
\]

Denote the time series of the possible for observer $i$ as $\mathbf{\overrightarrow{\mathbf{\Pi}}_{i,t}=}\{\mathbf{\varPi}_{i,t},\mathbf{\varPi}_{i,t+1},\mathbf{\varPi}_{i,t+2},...\}$.
Since OEE systems are innovative and unbounded in growth, they are
defined by a continual entry of new conditions to be considered by
a given observer. Therefore, $\mathbf{\overrightarrow{\mathbf{\Pi}}}_{i,t}$
is \textbf{open-ended} if the quotient of the possible at time $t+1$
and time $t$ is not empty, that is, if $\mathcal{P}_{i,t+1}\setminus\mathcal{P}_{i,t}\neq\emptyset$. 

Define the \textbf{adjacent possible} $A_{i,t}$ for any individual
as the next-in-sequence possible triplet, i.e., $A_{i,t}=\mathbf{\varPi}_{i,t+1}$
for individual $i$ at time $t$. The adjacent possible is a concept
developed by Stuart Kauffman (1993) and represents as-yet-unrealized-but-imaginable
``possibilities of possibilities'' that are reachable from $\mathcal{P}_{i,t}$.
Open-ended evolution is essentially the intended or unintended movement
into the adjacent possible. Unlike in CE knowledge theory, as we shall
discuss below, there is no way to logically deduce realizable paths
across OEE landscapes.

Consideration (4) localizes knowledge use and acquisition further,
stressing that \emph{perception in the form of theorizing and not
just observation is local to the embedded observer.} All theories
are accompanied by a set of undecidable disjunctions that an individual
must resolve one way or another when encountered in a choice context
(the existence of these undecidable disjunctions is an implication
of Gödelian incompleteness that we will discuss at greater length
below). Individuals may encounter different undecidable questions
and resolve them differently. Continual local theory revision thus
provides room in decision-making for nonlogical modes or context-informed
modes of reasoning as have been observed in real-world decision-making
(Todd \& Gigerenzer 2007; Smith 2003). Local theorizing also localizes
the (transaction) costs of decision-making and the calculation of
transaction costs.

In addition to the Considerations above we assume that individual
$i$ believes their logical systems $T_{i,t}$ to be consistent and
complete in that $T_{i,t}$ consistently and completely includes and
decides the truth value of all value statements that can be made about
the individual's known-world $\Omega_{i,t}$ (for formal definitions
of completeness and consistency, see the Appendix). This assumption
is analogous to the standard rationality assumption employed in economic
decision theories and was included in order to demonstrate that even
if we assume individuals are perfectly rational, knowledge use and
acquisition in OEE systems has a different character and set of implications
than the epistemic logic supporting CE rational choice models.

\subsection{A formalism for the construction of OEE knowledge}

As in CE knowledge theory, our OEE knowledge framework models individual
comprehension of the system as state perception and observation to
construct a cause-and-effect calculus and thus a theory of the known-world
that applies to all known possibilities. Individual $i$ expects the
real states they encounter $\omega\in\Omega$ to be descriptive of
their known-world $\Omega_{i,t}$ with respect to their current theory
of the known-world $T_{i,t}$. In a closed system, $T_{i,t}=T_{\Omega}$
and $\Omega_{i,t}=\Omega$, so all encountered states are automatically
consistent with the individual's theory. In OEE systems individuals
encounter sentences that are true in $T_{\Omega}$ that cannot be
proved true in $T_{i,t}$ (states that contain undecidable disjunctions).\footnote{Note that for OEE systems the presence of undecidable disjunctions
is not simply a result of Gödel (1931), as discussed in the Appendix,
but is implied by the definition of ``theory-breaking'' open-ended
evolution. The non-convergence of theoretical revision to $T_{\Omega}$,
however, requires Gödel (1931) for its proof.} This forces the individual to add to the set of known predicates
($\mathcal{P}_{i,t}\rightarrow\mathcal{P}_{i,t+1}$), modify their
language ($L_{i,t}\rightarrow L_{i,t+1})$, extend/revise/replace
their theory ($T_{i,t}\rightarrow T_{i,t+1})$, and change their conception
of the world ($\Omega_{i,t}\rightarrow\Omega_{i,t+1}$). Thus, from
the viewpoint of the observer-individual $i$, what is possible changes
($\Pi_{i,t}\rightarrow\Pi_{i,t+1}$). 

The procedure for constructing local knowledge at any point in time
is similar to construction of knowledge in traditional epistemology.
In order to define the local knowledge of an observer in an OEE system,
the observer uses a \textbf{decision procedure} $\delta_{it}$ on
the truth value of sentences $\xi$ defined within their theory $T_{i,t}$,
where $\delta_{it}$ is applied according to some model $\text{\textbf{M}}_{i,t}$
of the known-universe $\Omega_{i,t}$. Individuals construct a list
of 0's and 1's ordered in some manner (typically, alphabetically and
by length) of all possible true/false statements about the universe,
where 0 represents ``false'' and 1, ``true.'' If we suppose there
is a space $\Sigma_{it}=\{0,1\}^{\Omega_{i,t}}$ that lists 0 or 1
with respect to each unique sentence $\xi$ for each state $\omega\in\Omega_{i,t}$,
an individual $i$ employs a decision procedure $\delta_{it}:\Omega_{i,t}\rightarrow\Sigma_{i,t}$
that maps sentences in states to their truth values in $T_{i,t}$,
where
\begin{lyxcode}
\begin{equation}
\delta_{it}:\Omega_{i,t}\rightarrow\Sigma_{i,t},\delta_{i,t}(\xi)+\delta_{i,t}(\lnot\xi)=1,\forall\xi\in\mathcal{P}_{i,t}\label{eq:def-decision-function-1}
\end{equation}
\end{lyxcode}
The decision procedure $\delta_{it}$ is a simple function with a
constraint\footnote{The astute reader might realize that $\Sigma_{it}$ is essentially
Borel's number---or equivalently, Chaitin's $\Omega$---for the
individual's known-universe $\Omega_{i,t}$ (Chaitin 2005). $\delta_{it}$,
then, queries the ``Delphic oracle.''}, defined for a complete and consistent theory $T_{i,t}$ and thus
defined over all $\xi\in\Omega_{i,t}$ (this is stated and proved
as Proposition \ref{prop:delta-defined-all-sentences} in the Appendix).

As in traditional epistemology, $\mathfrak{K}_{i,t}$ is a list of
sentences $\xi$ that have been deemed true or false with respect
to a decision procedure $\delta_{it}$ on $T_{i,t}$ according to
some model $\text{\textbf{M}}_{i,t}$ of the known-universe $\Omega_{i,t}$,
where
\begin{lyxcode}
\begin{equation}
\mathfrak{K}_{i,t}:=\{\xi\in\Omega_{i,t}:\delta_{i,t}(\xi)+\delta_{it,}(\lnot\xi)=1\}\label{eq:def-knowledge-list-2}
\end{equation}
\end{lyxcode}
By the rationality assumption, individuals expect any state $\omega$
they encounter to be complete, consistent, and to contain all predicates
in $\mathcal{P}_{i,t}$. The states the individual believes to be
within the realm of possibility at time $t$--defined as the individual's
\textbf{contextual knowledge possible}---is the set of all $\delta_{it}$-decidable
states 
\begin{lyxcode}
\begin{equation}
\mathcal{K}_{i,t}:=\{\omega\in\Omega_{i,t}:\delta_{i,t}(\omega)\in\Sigma_{it}\}\label{eq:def-knowledge-possible-2}
\end{equation}
\end{lyxcode}
We can then define an observer $i$'s \textbf{local knowledge} $\kappa_{i,t}(\omega)$
at time $t$ as the set of all sentences pertaining to a given state
$\xi\in\omega$ that start with $k_{i,t}$, which we can obtain by
prepending all sentences $\xi\in\mathfrak{K}_{i,t}$ with $k_{i,t}$.

Defining $\kappa_{i,t}$ allows us to relate an individual's knowledge
directly to the system state-space $\Omega$. Unlike in traditional
epistemology (cf. Section \ref{subsec:Traditional-epistemology:-formal}),
individuals are not granted knowledge of the simplest set of predicates
that completely generates the theory $T$ of the system $\Omega$.
While $\kappa_{i,t}$ is defined on $\Omega$, it is constrained by
being constructed from $\mathfrak{K}_{i,t}$ and ultimately from $T_{i,t}$
where the open-endedness of the system implies that $T_{i,t}\neq T_{\Omega}$. 

The observer $i$ uses their local knowledge $\kappa_{i,t}$ to generate
a cause-and-effect partition $\mathscr{I}_{i,t}$ of their universe,
where $\mathscr{I}_{i,t}(\omega)=\{\omega'\in\Omega:\kappa_{i,t}(\omega)=\kappa_{i,t}(\omega')\}$.
As in traditional epistemology, this partition patterns events as
implying certain characteristics and observations about the world
in the form of states, and answers questions like: ``Will it rain
today?'' The dynamics of altering the partition in the face of the
unlistable novelty of OEE systems differ fundamentally from the recalculation
of weights in Bayesian-style dynamic updating of partitions where
all possibilities have been listed. In an OEE system, individual $i$
must determine what constitutes their state-specific and non-state-specific
knowledge at each step in time. At time $t$, individual $i$ forms
hypotheses and gathers ``facts''. ``Facts'' are hypotheses $i$
believes to have been correctly inferred or deduced in theory $T_{i,t}$
according to a model $\text{\textbf{M}}_{i,t}$ of the known-universe
$\Omega_{i,t}$. Hypotheses are as-yet-undecided statements with an
unknown true/false valuation. This means that $i$ must employ a decision
procedure in $T_{i,t}$ that can settle the truth of any new predicate
so it can enter into their knowledge.

Open-ended evolution thrusts $i$ into a state-space $\Omega_{i,t+1}$
which contains predicates $P\in\mathcal{P}_{i,t+1}$ and thus sentences
$\xi'\in\omega'$ for $\omega'\in\Omega_{i,t+1}$ outside the domain
of $\delta_{i,t}$. How, then, does an individual update her local
knowledge as the system evolves in an open-ended manner? 

Any states that fall outside the realm of possibility $\mathcal{K}_{i,t}$
for the individual must contain sentences $\xi'$ for which $\delta_{i,t}$
cannot decide the truth value. Define the\textbf{ adjacent knowledge
possible} $\mathcal{A}_{i,t}$ for individual $i$ at time $t$ as
the quotient of the adjacent ``knowledge possible'' $\mathcal{K}_{i,t+1}$
with the current possibility set $\mathcal{K}_{i,t+1}$, where $\mathcal{A}_{i,t}\equiv\mathcal{K}_{i,t+1}\setminus\mathcal{K}_{i,t}=\{\omega':\omega'\in\mathcal{K}_{i,t+1},\omega'\notin\mathcal{K}_{i,t}\}$. 

In OEE systems, the state-level adjacent possible $\mathcal{A}_{i,t}$
is non-empty (Proposition \ref{adjacent-nonempty}), which follows
from how we defined OEE systems. 

The two components of an individual's possibility space in OEE systems
are generated by an individual's \emph{contextual} and \emph{temporal}
local-ness. The contextual knowledge possible $\mathcal{K}_{i,t}$
is generated by freeze-framing an individual's evolution and is equivalent
to the definition of the possible in traditional epistemology (Samet
1990: 193). The adjacent knowledge possible $\mathcal{A}_{i,t}$ is
generated by moving between one known-world $\Omega_{i,t}$ and the
next $\Omega_{i,t+1}$. $\mathcal{A}_{i,t}$ is defined for $\omega_{i,t+1}$.

A key part of obtaining predictive dynamics in CE knowledge theory
is establishing that states of the world $\omega\in\Omega$ as seen
by the individual are consistent (see Definition \ref{def-consistent}),
coherent (see Definition \ref{def-coherent}) and complete (see Definition
\ref{def-complete}) (Aumann 1999a: 276, Samet 1990). States $\omega\in\Omega_{i,t}$
in OEE systems do not in general exhibit such analytically nice properties
with respect to all $\omega\in\Omega$. While states in OEE systems
are defined as \emph{locally} consistent, complete and coherent and
are provably coherent at the individual level (for all states $\omega\in\Omega$),
they are not consistent or complete at the system level (see Theorem
\ref{coherence-local-knowledge} in the Appendix). The system-level
coherence of local knowledge in OEE systems is implied by the assumption
of local completeness, local consistency and local coherence without
the need to assume individuals possess the full theory of the universe.
Coherence allows individuals to categorize phenomena and make decisions
consistent with their current theoretical understanding. An implication
of the coherence of local knowledge is that in slower-changing contexts,
individuals can engage in cross-sectional error correction and iteratively
progress towards a better understanding of their known-world. Faster-changing
contexts degrade the relative efficacy of error correction enabled
by local coherence.

Next we demonstrate that in OEE systems, within-observer knowledge
is fundamentally incomplete and between-agent knowledge is fundamentally
disjoint. Consider the following example. An observer $i$ can only
distinguish between two states in their known-world if the local knowledge
sets of those two states are different. If their cause-and-effect
partition associates the observation ``it is cloudy'' to a particular
state but not to another, then the knowledge that ``it is cloudy''
allows them to distinguish between the two states. However, if they
cannot distinguish between the two states then either it is not cloudy
or the observer can't tell (doesn't know) if it is cloudy.

Cause-and-effect partitions $\mathscr{I}_{i,t}$ of $\Omega_{i,t}$
are defined with respect to the contextual knowledge possible $\mathcal{K}_{i,t}$.
Due to the OEE nature of the system (and Gödelian incompleteness)
there will always exist a sentence $\xi$ that is possible in $\Omega$
but not decidable by the decision procedure $\delta_{i,t}$ (see Proposition
\ref{adjacent-nonempty} and Lemma \ref{undecidable-sentence-exists}).
Therefore, individual-level partitions of $\Omega$ must always be
incomplete. 

Heterogeneity in $\mathcal{K}_{i,t}$ can be explained by computational
complexity, open-endedness, individual characteristics like entrepreneurial
ability and alertness, and endogenously random factors (the creative
importance of time and place). Felin (2022) describes the actual environment
of human choice as a rich unmappable landscape, whereby individuals
are embedded in ``{[}o{]}rganism-specific, teeming environments.''
Individual worlds $\Omega_{i,t}$ determine local knowledge and, in
OEE choice contexts, $\Omega_{i,t}\ne\Omega_{j,t}$. Thus, different
individuals in OEE systems have different minds. In CE systems, different
individuals have one mind. We prefer many-minds theorizing to one-mind
theorizing.
\begin{conjecture}
(Disjointness) Under open-ended evolution,  local knowledge is at
least partially disjoint at any cross section of time and between
cross sections. That is, $\mathcal{K}_{i,t}\ne\mathcal{K}_{j,t}$
in general for $i,j\in N,t\in T$.
\end{conjecture}
\begin{proof}
(Sketch) Suppose individuals $i$ and $j$ perceive the same possibility
spaces $\mathcal{K}_{i,t}=\mathcal{K}_{j,t}$. Then, this implies
they share theories of the universe $T_{i,t}=T_{j,t}$ and that they
perceive the same known-worlds $\Omega_{i,t}=\Omega_{j,t}$. Suppose
$i$ encounters a new predicate in their adjacent possible that $j$
does not. Then, at time $t+1$, $T_{i,t+1}\neq T_{j,t+1}$. This is
not too much of a problem if  $j$ updates their theory in the same
way, but this likely requires theory convergence in the face of vast
combinatorial complexity generated by open-ended processes. As proved
in the Appendix (Theorem \ref{Tplus-essential-extension}), we cannot
generally claim theory convergence in OEE systems. Therefore, it is
reasonable to claim the general partial disjointness of individual
knowledge in OEE systems.
\end{proof}
The incompleteness of knowledge partitions in OEE systems constrains
what we can say about the strategic arithmetic of interactions between
individuals. In general, $\mathcal{K}_{i,t}\ne\mathcal{K}_{j,t}$
for $i,j\in N,t\in T$. The Disjointness Conjecture implies that however
individuals coordinate with each other to realize common social goals,
it is not through the automatic knowing of the needs, worldviews and
goals of others. This suggests a possible role for institutions like
religion, culture, and aesthetics to encode worldviews individuals
can adopt in common. We discuss how institutions emerge from our OEE
knowledge framework in more detail in Section \ref{subsec:Nonlogical-search-institutions}
and in the Conclusion.

\section{\label{sec:Interpretation-and-implications}Some results and implications}

\subsection{Theory non-convergence in OEE systems (OEE1)}

Suppose we assume that the true universe $\Omega$ is consistent and
closed and described by a theory-of-everything $T_{\Omega}$. Let's
take up Result OEE1 (stated at first as a claim to be proved) and
ask: if we allow the open-ended process of theory revision to continue
ad infinitum, can individual $i$ infer theory $T_{\Omega}$ in a
finite amount of time?

Open-ended processes, in our definition, represent theoretical revisions.
Can individuals in open-ended evolution somehow subvert Kuhn and Feyerabend
and at time $t$ infer a path of theoretical revision $\mathbf{\overrightarrow{\mathbf{T}}}_{i,t}^{\Omega}=\{T_{i,t}\rightarrow T_{i,t+1}\rightarrow...\rightarrow T_{\Omega}\}$
that progresses towards the theory-of-everything $T_{\Omega}$? 

The above was a central question the the wake of the incompleteness
proofs of Gödel (1931), Rosser (1936), and Post (1944). Alan Turing
(1939) attempted to circumvent Gödel incompleteness by constructing
a sequence of logical languages obtained through sequentially recursive
extensions. His conclusion was that it is impossible to find ``a
formal logic which wholly eliminates the necessity of using intuition''
and that the mathematician must instead ``turn to 'non-constructive'
systems of logic with which not all the steps in the a proof are mechanical,
some being intuitive'' such that ``the strain put on intuition should
be a minimum'' (Turing 1939: 216). 

Even if the theoretical progression $\mathbf{\overrightarrow{\mathbf{T}}}_{i,t}^{\Omega}$
of each individual $i$ through the adjacent possible of an OEE system
is a process of incrementally and consistently extending some initial
theory $T_{i,t}$, we cannot in general conclude that each theoretical
innovation is derivable from the theory that came before due to the
novelty-generating qualities of an OEE system which have the tendency
to outgrow old theories in unpatternable ways (Theorem \ref{Tplus-essential-extension},
stated and proved in the Appendix). This result implies that an individual
at time $t$ has no access to future theoretical discoveries represented
by the time series $\mathbf{\overrightarrow{\mathbf{T}}}_{i,t+1}$
or to the time series of future worlds $\mathbf{\overrightarrow{\mathbf{\Omega}}}_{i,t+1}$.
Theoretical innovation is a process of observer-specific ``becoming''
into new possibilities through time rather than atemporal reflection.

\subsection{Common knowledge under open-ended evolution (OEE2 \& OEE3)}

An information set under open-ended evolution is the set of states
in $\Omega_{i,t}$ that individual $i$ cannot distinguish from one
another, or 

\[
\mathbf{I}(\omega):=\{\omega'\in\Omega_{i,t}:\kappa_{i,t}(\omega')=\kappa_{i,t}(\omega)\}
\]

The information partition $\mathscr{I}_{i,t}$ is the partition of
$\Omega_{i,t}$ formed by the set of all information sets over the
states of $\Omega_{i,t}$. Given that $\Omega_{i,t}\ne\Omega_{j,t}$
in general (the Disjointness Conjecture), we cannot say individuals
know the information partitions $\mathscr{I}_{i,t}$ of other individuals.
Different known-worlds $\Omega_{i,t}$ imply different theories $T_{i,t}$:
individuals will not in general have access to the same ``dictionary''
of states. 

We're now ready to address Result OEE2: \emph{Common knowledge is
generally impossible in an open-ended evolutionary system.}

In traditional epistemology, a knowledge hierarchy of the world is
formed by considering all possible relevant states, what $i$ knows
about the true state of nature and about what $j$ knows about the
true state of nature, what $j$ knows about the true state of nature
and what $j$ knows about what $i$ knows, what each know about what
each knows, and so forth. A knowledge hierarchy for individual $i$
is an infinite sequence of sets of states representing this process.
The contextual knowledge possible $\mathcal{K}_{i,t}$ describes the
range of the possible at time $t$ according to each individual $i$.
We can also describe knowledge hierarchies as in Aumann (1999a: 294)
based on the set of all possible states that describe some aspect
of reality (states in which it will be cloudy in New York City on
December 1). 

There is a vast universe of combinatorial possibilities for such a
sequence, so the feasibility of a knowledge hierarchy $h_{i}$ requires
that $i$ and $j$ consider the same set of possible states of nature
relevant to some aspect of reality. In traditional epistemology, both
$i$ and $j$ perceive the same known-world $\Omega$ and thus generate
the same theory of the world $T$ based on the same set of nonlogical
constants $P$. This implies that the contextual knowledge possible
is the same for each individual, as knowledge has been essentially
decontextualized: all state-spaces are equivalent at the multiagent-
and system-level $\Omega_{i,t}=\Omega_{j,t}=\Omega$. Thus, all theories
are equivalent $T_{i,t}=T_{j,t}=T_{\Omega}$. 

In traditional epistemology, as there is no open-ended evolution,
there is no fundamentally local knowledge. Knowledge hierarchies are
mutually consistent as an artifact of the closed-endedness of the
model environment. States are defined in traditional epistemology
as mutually consistent pairs of hierarchies, and the universe as the
set of all such states. Such definitions make common knowledge possible
within the theory.

Common knowledge is not, however, a natural state of affairs in OEE
systems. In OEE systems, pairs of hierarchies $(h_{i,t},h_{j,t})$
are, in general, not mutually consistent (Proposition \ref{hierarchies-not-mutually-consistent},
stated and proved in the Appendix), and therefore common knowledge
is generally impossible under open-ended evolution (Corollary \ref{Common-knowledge-impossible},
stated and proved in the Appendix). This demonstrates Result OEE2.
It is a simple matter to then demonstrate that it is possible for
individuals to agree to disagree under open-ended evolution as a direct
consequence of Proposition \ref{hierarchies-not-mutually-consistent}
and Corollary \ref{Common-knowledge-impossible} (Corollary \ref{agree-to-disagree},
stated and proved in the Appendix). This demonstrates Result OEE3.

\subsection{Open-ended evolutionary epistemology and ``frame relativity''}

OEE epistemology indicates a deeply constrained relationship between
the embedded observer and what they knows about their universe. We
must grapple with the problem of radical uncertainty in OEE systems.
We will not solve this problem if solving means ``make tractable''--there
is no making radical uncertainty tractable. Individual in OEE systems
can and do make choices in the face of radical uncertainty, requiring
any reasonable OEE knowledge theory to describe in some fashion how
individuals know what they cannot know. 

In particular, individuals must contend with undecidable disjunctions.
Kurt Gödel's (1931) famous incompleteness theorem, which upended Hilbert's
program to reduce all of mathematics to a finite set of axioms, proved
in all theories with a basic level of arithmetic that there exists
true sentences that cannot be proved true within the theory\footnote{Strictly speaking, Gödel proved this result under the assumption of
\textquotedbl omega consistency\textquotedbl{} whichis a weaker condition
than consistency. The theorem as stated was first proved by Rosser
(1936), although the roughly simultaneous results of Church (1936)
and Turing (1937) may be used to prove Rosser's result.}. In our OEE knowledge framework, this translates into the incompleteness
of any given individual's theory, i.e., for individual $i$ at time
$t$ with a theory $T_{i,t}$ of their known-world $\Omega_{i,t}$,
there exists a sentence $\xi'\in\Omega$ whose truth value is undecidable:
$\delta_{i,t}(\xi')+\delta_{it,}(\lnot\xi')\neq1$. That is, no observer
in an OEE system can have a complete model of the universe (stated
and proved as Proposition \ref{Frame-relativity-A} in the Appendix).

Can an individual engage in iterative theorizing, deciding observed
undecidable disjunctions when encountered in a way that gets them
eventually to the true theory of the universe $T_{\Omega}$? This
is essentially asking if there exists a time series $\overrightarrow{\mathbf{\delta}}_{i,t}=\{\delta_{t},\delta_{i,t+1},\delta_{t+2},...\}$
that converges to the actual decision process of the OEE system, $\delta_{\Omega}$,
or alternately, a time series of theories $\mathbf{\overrightarrow{\mathbf{T}}}_{i,t}=\{T_{i,t},T_{i,t+1},T_{i,t+2},...\}$
that converges to the actual theory of the OEE system, $T_{\Omega}$.
For this to be true, $T_{\Omega}$ would have to be decidable. In
OEE systems, however, $T_{\Omega}$ is not decidable, as a consequence
of Gödel's ICT and a number of more technical results of our OEE knowledge
framework (see the Appendix for Theorems \ref{Decidable-essentially-not-axiomatizable},
\ref{oeee-is-undecidable}, \ref{godel-ICT} , Proposition \ref{m-does-not-halt},
and Corollary \ref{Frame-relativity-b}).

Not only is $T_{\Omega}$ undecidable by any observer $i$, but there
is no end to the number of statements that are true in $T_{\Omega}$
but whose truth value cannot be ascertained in $T_{i,\Omega}$ (Lemma
\ref{dense-undecidable-disjunctions}, stated and proved in the Appendix).
Of course, no individual knows the full set of undecidable statements,
as they are unlistable and, we conjecture, not recursively enumerable
(Kauffman \& Roli 2021). 

We summarize these sets of results as \textbf{frame relativity}, namely,
that a complete and correct model or theory of an open-ended evolutionary
system is possible only for agents existing outside the system and
impossible from within the system. Neither can an individual know
all the ways in which their theories of the world can't account for
possibilities realizable in $\Omega$, but since there is no process
to generate all possibilities realizable in $\Omega$, the observer
is inextricably bounded by the ``frame'' of their known-world and
the theory they have developed to explain their known-world. Frame
relativity is the great epistemological equalizer: all people are
bound to their frame, regardless of their education, experience and
position in life.

\subsection{\label{subsec:Nonlogical-search-institutions}Nonlogical search in
OEE systems (OEE4)}

As in Section \ref{sec:A-primer-on-open-ended-evolution}, assume
that individuals in OEE systems believe their theories $T_{i,t}$
of their known-worlds are complete and epistemically closed until
they encounter a statement whose truth value is undecidable in their
theories. Knowledge in OEE systems is a partition $\mathscr{I}_{i,t}(\omega)$
of local knowledge $\kappa_{i,t}$ according to a decision procedure
$\delta_{i,t}$. But how does this partition come about? The construction
$\delta_{i,t}$ masks the underlying process of encountering valid
statements and deciding upon their truth values in various contexts,
i.e., decision-making under uncertainty. As individuals in OEE systems
are continually constructing new $\delta_{i,t}$, it is vital to explain
the contours of this process. 

Decision-making in economics is typically formalized as Bayesian subjective
expected utility theory, whose conclusions have come under scrutiny
in light of results from economic experiments (Smith 2003) which suggest
that cognitive and computational difficulty of a decision (Kahneman
\& Tversky 1973), the knowledge context in which a decision is made
(Cox \& Griggs 1982; Gigerenzer \& Hoffrage 1995; Rizzo \& Whitman
2009) strongly effect how individuals make decisions. Open-endedness
throws another wrench into the decision problem, where the unknowability
of the possibility space in question becomes a significant factor
in decision-making. While spending more time in search generally corrects
cognitive/computational difficulties and improves the knowledge context
of decision-making, spending more time in search in an OEE system
means a higher likelihood of encountering novel possibilities, of
having one's known-world change. 

Knowledge in OEE systems is fragmented: it isn't held in common and
individuals can agree to disagree (by Results OEE2 \& OEE3). Generally,
however, individuals still rationally benefit from sharing knowledge
in coordinative interactions. Since we do not get shared knowledge
for ``free'' in OEE systems, it stands to reason that part of solving
coordination problems in OEE systems requires the explicit construction,
spread, and maintenance of methods for sharing and updating knowledge
in OEE systems. Thus, the epistemic need for institutions emerges
from our OEE knowledge theory framework, and from a scientific perspective,
the necessity to take seriously institutional and cultural evolution.

While describing how institutions and cultural technologies emerge
and evolve is outside the scope of this paper, we can use our framework
to get an abstract sense for how these processes unfold in OEE systems
in response to epistemic necessity. Institutions will emerge as an
epistemic palliative to the coordination problem, and their number
and character will depend on the degree of knowledge fragmentation,
their entanglement with the adjacent possible, and on-the-ground particulars
involving the specifics of how the OEE system has evolved through
time\footnote{In this paper we neglect non-epistemic inducements to the formation
of institutions.}. 

Knowledge fragmentation in OEE systems exacerbates between-individual
differences, which are not only subjective but experiential and theoretical.
Several institutions could emerge to serve the same set of needs among
different knowledge ``niches'' in a system. In this context, niches
are subsystems or groups of individuals with a greater frequency of
interaction and a greater degree of shared knowledge. Furthermore,
institutions in OEE systems must be robust to novel and unpredictable
changes in the system. Since strictly rational systems will ``break''
when the consensus theory of a knowledge niche is updated, \emph{robust
institutions in OEE systems cannot be strictly rational if they are
to survive}. 

Similarly, individual decision-making in OEE systems will tend to
transcend the set of strictly rational possibilities available to
individuals at any time $t$. In practical terms, $\delta_{i,t}$
represents a collection of decision procedures based on a collection
of task- and environment-specific models and theories (Felin \& Koenderink
2022). But $\delta_{i,t}$ cannot in general decide all relevant truths
given any particular situation under open-ended evolution, as shown
above. Interactions under fragmented knowledge are characterized by
the theories and known-worlds of individuals being different: $T_{i,t}\ne T_{j,t},\Omega_{i,t}\ne\Omega_{j,t}$.
Under open-ended evolution, therefore, ``rationality'' is no longer
presumptively coordinative. Still, individuals benefit if they can
coordinate their plans with other individuals, implying the incentive
to create and adopt coordinative social structures like markets, legal
standards, governance systems, philosophies and religious codes. These
social structures may not--will not---perfectly substitute for rational
reasoning as if an individual possessed the correct theory of the
world $T_{\Omega}$, and often involve adopting a perspective of the
world individual $i$ may not completely agree with or believe possible.

Beyond coordination, cultural and social institutions can also be
means of accessing truths that lie in the adjacent possible. Suppose
the members of Group A only employ a rational knowledge generating
mechanism based upon a theory $T_{A,t}$ which is itself a complete
and consistent theory of known-world $\Omega_{A,t}$. Suppose there
is a true statement of the universe that isn't provable true in $T_{A,t}$.
Consider Group B whose members employ a mixture of rational mechanisms
based on $T_{A,t}$ but also utilize a nonlogical mode of reasoning
like aesthetics about the true statement inaccessible to Group A.
Then, with respect to decisions that benefit from engaging with the
trust statement inaccessible to Group A, Group B will have a coordinative
advantage. 

The above discussion generally demonstrates Result OEE4: \emph{Nonlogical
(random, heuristic, aesthetic) searches in OEE systems can be as knowledge
generative as logical search mechanisms}.

Nonlogical search protects individuals and groups from getting stuck
searching under street lights, though it doesn't guarantee an optimal
solution to any particular problem. The benefit of nonlogical systems
like aesthetic movements is that they have an internal logic that
grants a way to systematically search outside of street lights. The
search for symmetry and ``elegance'' in physics has yielded many
insights, but is largely aesthetic (not deducible from the mathematical
theory underlying physics). The prolific mathematician Poincaré believed
that a mathematician uses their aesthetic sensibilities as ``a delicate
sieve'' on choice, without which they can ``never be a real creator''
(Poincaré 1920: 28-9).

Given the infinitude of possible combinations facing any chooser and
the infinitude our imaginations about what might be possible in OEE
systems, heuristics like aesthetics allow us to choose in systematic
ways and justify the (non)logic of our choice with others who comprehend
our aesthetic (or other heuristic) values (Devereaux et al 2024; Todd
\& Gigerenzer 2007; Smith 2003). We can access truths using these
heuristics that are impermeable to logical modes of reasoning, though
we cannot rank heuristic search methods objectively.

\subsection{The non-ergodic nature of knowledge in OEE systems (OEE5)}

The ergodic theorem (Birkhoff 1931) states that there exists a probability
that a point in any trajectory defined for a a manifold lies in a
given volume of the manifold; that is, that one can define a probability
distribution over all attainable points in the system. 

Peters (2019) restates this relationship as 

\begin{equation}
\lim_{T\rightarrow\infty}\frac{1}{T}\intop_{0}^{T}f(\omega(t))dt=\intop_{\Omega}f(\omega)P(\omega)d\omega\label{eq:ergodic-theorem-states}
\end{equation}
where the left side is the time average of $f$, and the right side
is the expectation value of $f$. $\Omega$ has the same meaning as
in the modal logic of this paper, as the collection of all possible
system states. That is, any time-dependent trajectory through the
state space can be modeled as a function with probabilistic weights
over all states. We can easily analogize this relationship to consider
knowledge trajectories as probability distributions.

Suppose Nature is ergodic. This isn't too wild of a proposition, as
it is possible that 

\begin{equation}
\lim_{T\rightarrow\infty}\frac{1}{T}\intop_{0}^{T}\kappa_{i,t}(\omega)dt=\intop_{\Omega}\kappa(\omega)P(\omega)d\omega\label{eq:ergodic-knowledge}
\end{equation}
Is open-ended knowledge about Nature ergodic? The answer is, simply,
no. We can demonstrate this formally.
\begin{prop}
Knowledge in open-ended evolutionary systems is non-ergodic. 
\end{prop}
\begin{proof}
Suppose knowledge in open-ended evolution is ergodic. We will prove
that this implies that the system cannot be open-ended. If knowledge
in open-ended evolution is ergodic, then Equation \ref{eq:ergodic-theorem-states}
must apply to local knowledge as defined by Definition \ref{def-local-knowledge}.
An individual $i$'s local knowledge is $\kappa_{i,t}(\omega)$, the
set of all sentences in $\omega$ that start with $k_{i,t}$. Then,
$\lim_{T\rightarrow\infty}\frac{1}{T}\intop_{0}^{T}\kappa_{i,t}(\omega)dt=\intop_{\Omega}\kappa(\omega)P(\omega)d\omega$.
As $\kappa(\omega)$ is timeless, deriving it in any time period $\tau$
requires apprehending all of $\Omega$. But by Proposition \ref{adjacent-nonempty},
$\mathcal{A}_{i,\tau}\equiv\mathcal{K}_{i,\tau+1}\setminus\mathcal{K}_{i,\tau}=\{\omega'\notin\mathcal{K}_{i,\tau}\}$
is nonempty, meaning there exist states that lie in $\mathcal{K}_{i,\tau+1}$
that $i$ cannot integrate over.

Breaking the right hand side of Equation \ref{eq:ergodic-knowledge}
into its component pieces, the integral over knowledge up until $\mathcal{K}_{i,\tau}$
and the integral of knowledge over $\mathcal{K}_{i,\tau+1}$ and beyond
becomes: 

\begin{equation}
\lim_{T\rightarrow\infty}(\frac{1}{T}\intop_{0}^{\tau}\kappa_{i,t}(\omega)dt+\frac{1}{T}\intop_{\tau}^{T}\kappa_{i,t}(\omega)dt)=\intop_{\Omega}\kappa(\omega)P(\omega)d\omega\label{eq:nonergodic-knowledge-breakdown}
\end{equation}

But $i$ only has access to the first part of that equation, $\frac{1}{T}\intop_{0}^{\tau}\kappa_{i,t}(\omega)dt$.
Therefore, they can never derive the right-hand side of the equation
unless $\mathcal{A}_{i,\tau}$ is empty---i.e., the system is closed.
\end{proof}
In his criticism of ergodic economic theory in the form of expected
utility theory, Peters (2019) observes that ``...in maximizing the
expectation value --- an ensemble average over all possible outcomes
of the gamble --- expected utility theory implicitly assumes that
individuals can interact with copies of themselves, effectively in
parallel universes (the other members of the ensemble). An expectation
value of a non-ergodic observable physically corresponds to pooling
and sharing among many entities. That may reflect what happens in
a specially designed large collective, but it doesn\textquoteright t
reflect the situation of an individual decision-maker.'' 

We can demonstrate Peters' statement for epistemology, using the infrastructure
of this paper. First of all, observe that the left-hand side of Equation
\ref{eq:ergodic-knowledge}, $\lim_{T\rightarrow\infty}\frac{1}{T}\intop_{0}^{T}\kappa_{i,t}(\omega)dt$,
does not converge under open-ended evolution. Suppose it does converge
to some set $\kappa_{i}(\omega)$. This implies there is no sentence
$\xi$ that is undecidable in $\kappa_{i}(\omega)$. But there is
always an undecidable sentence $\xi$. By Proposition \ref{m-does-not-halt},
no process for fully mapping knowledge halts, and there is no progressive
sequence of theorizing available to any individual. But this means
that there is no unique theory sequence from theories $T_{\Omega_{i}}$
to $T_{\Omega'_{i}}$, where $T_{\Omega_{i}}\neq T_{\Omega'_{i}}$.
Suppose there exists an $\xi\in\Omega_{i}$ not decidable in $T_{\Omega_{i}}$
that is decidable in $T_{\Omega'_{i}}$. Then, in order for individual
$i$ to decide $\xi$ in $\text{\ensuremath{\Omega_{i}}}$, it would
need access to $\Omega_{i}^{'}$. But there is no algorithmic way
to access $\Omega_{i}^{'}$ from $\Omega_{i}$ by Proposition \ref{m-does-not-halt}.
Therefore, constructing an expected utility function that decides
weights for an event based on the sentence $\xi$ requires, in effect,
access to another self with theory $T_{\Omega_{i}^{'}}$ in another
universe $\Omega_{i}^{'}$.

\subsection{How frame relativity constrains prediction in OEE systems}

Time series in our OEE knowledge framework can be complicated, as
the scalar value of the increment $t$ is neither fixed for the individual
nor between-individuals. Note especially that were are talking about
the time series of OEE knowledge and not the time series of system
and individual behavior, which could be quite rich within each time
increment, and time increments themselves may be long relative to
the time increments for which system and individual behavior are typically
defined. 

Time increments are defined by an individual updating their theory
of the world $T_{i,t+1}\neq T_{i,t}$. A more consistent way of defining
an increment within-individuals is to define the OEE time series $\mathbf{\overrightarrow{\mathbf{T}}}_{i,t}=\{T_{i,t_{1}^{i}},T_{i,t_{2}^{i}},...\}$
where $t_{k}^{i}$ is the time increment wherein individual $i$ updates
their theory of the world for the $k$-th time. 

Defining the time series of an OEE system is not as straightforward.
Not all individuals may be aware of all the changes happening in a
system, nor of the knowledge updating of other individuals in the
system. Some individuals may update their knowledge about observations
previously observed by and used to update the knowledge of other individuals,
as there is no way to presume common knowledge and individuals who
do communicate can agree to disagree.

The challenge in defining the time series of an OEE system is in defining
a system-level time increment to track the (who-knows-what) time series
of the system. We suggest that for modeling purposes, this can be
done fairly arbitrarily. Pick some increment $\tau$ such that $|\tau|=|\tau_{1}|=|\tau_{2}|=...=|\tau_{k}|=...$
. An OEE system will not update its who-knows-what time series for
some $\tau_{k}$, and it will for others. So the increment of the
knowledge time series of an OEE system where each increment definitionally
updates the knowledge in the system in an open-ended way is a ragged
partition with differently-sized bins of the set $\mathcal{T}=\{\tau_{1},\tau_{2},\tau_{3},...\}$. 

Clearly, these bins can be defined after-the-fact. Defining the bins
ahead of time would require intimate foreknowledge of not only the
open-ended behavior of the system and individuals, but also the open-ended
knowledge process of every single individual in the system. No observer
embedded in an OEE system can have access to this level of knowledge,
as our frame relativity result demonstrates. Therefore, there is a
limit on what embedded observers can say about how human systems evolve.
This first implication of frame relativity on prediction predicts
expert failure. This implication can be extended to any subset of
individuals in the system, constraining the prediction power of any
given ``consensus'' of thought. Thus we can extend the constraints
on the use of knowledge as argued in Hayek (1945) to scientific consensus. 

\section{Conclusion: reflections on economic thought \& theory}

The closed-ended version of epistemic logic assumes away many differences
between individuals. In reality, we have different theories. We live
in different worlds. The tick and tiger live in different worlds with
different event spaces. The tick feeds on the tiger, but does not
know what a tiger is. The tiger scratches at the tick, but does not
know what a tick is.The tick waits in high place such as a tall blade
of grass and drops down to feed when it smells butyric acid, the telltale
sign of a mammal. The tiger, instead, seeks large prey such as gazelle.
To us they live in the same world. But the event space of the tick
has nothing in common with the event space of the tiger. In this sense,
they live in different worlds. And if William James (1890), Alfred
Schütz (1945), and Jakob von Uexküll (1934) were right, different
people live in different worlds as well. Individuals themselves live
in different worlds at different points in time, thanks to nonergodic
change. In the theories of Aumann and other practitioners of standard
epistemic logic, we all live in one world, common to us all. The epistemic
logic of our OEE knowledge framework allows the same person to live
in different worlds at different times.

Economics as a discipline turns on knowledge, its questions inherently
involving the costs of knowledge acquisition, and what can and can't
be known theoretically. Adam Smith (1751, 174.9) describes the overwhelming
cost of instituting a deterministic normative system in a complex
world\footnote{``The general rules of almost all the virtues...are in many respects
loose and inaccurate, admit of many exceptions, and require so many
modifications, that it is scarce possible to regulate our conduct
entirely by a regard to them'' . (Smith 1751, 174.9)} . Knut Wicksell (1898 {[}1936{]}) warned against too much precision
in the new mathematical economics of marginal analysis, subtly referring
to the complexity of the social world as his reason\footnote{``I have on this occasion made next to no use of the mathematical
method. This does not mean that I have changed my mind in regard to
its validity and applicability, but simply that my subject does not
appear to me to be ripe for methods of precision. In most other fields
of political economy there is unanimity concerning at least the \emph{direction}
in which one cause or another reacts on economic processes; the next
step must then lie in an attempt to introduce more precise quantitative
relations. But in the subject to which this book is devoted the dispute
still rages about \emph{plus} as opposed to \emph{minus}.'' (Wicksell
1898 {[}1936{]}: xxx)} . Frank Knight (1921) distinguished between (predictable) risk and
(unpredictable) uncertainty. F. A. Hayek (1937, 1945) presages the
necessity of understanding the process of knowledge acquisition (1937:
33) as one of unprestatable observations that cannot be deducible
from our existing theory of the world (ibid: 36), the fundamental
disjointness of knowledge\footnote{``There would of course be no reason why the subjective data of different
people should ever correspond unless they were due to the experience
of the same objective facts'' (Hayek 1937: 43).} and the uncrossable distance between individual- and system-level
knowledge\footnote{``The equilibrium relationships cannot be deduced merely from the
objective facts, since the analysis of what people will do can only
start from what is known to them. Nor can equilibrium analysis start
merely from a given set of subjective data, since the subjective data
of different people would be either compatible or incompatible, that
is, they would already determine whether equilibrium did or did not
exist. `` (Hayek 1937: 43)} (ibid: 43). The title of this paper is an homage to F. A. Hayek's
(1945) ``The use of knowledge in society.''

Despite economists like Hayek, the focus of mid-20th century mathematical
economics shifted from contemplating dynamical complexity and unknowability
to conducting exercises in neoclassicism's fixed-point analysis, manifested
in tools like linear programming, value theory, control theory, subjective
probability theory, and traditional epistemology. In the 1970s several
key premises of the neoclassical program were called into question
or outright disproved (Sonnenschein 1973; Mantel 1974; Debreu 1974;
Lewis 1985; cf. Rizvi 2006 for an overview) and the contemplation
of dynamical complexity and unknowability saw something of a revival
(Shackle 1972; Lachmann 1976; Kirzner 1973) amid an attempt to rescue
portions of the neoclassical paradigm in the form of the rational
subjective expectations theory.

In an OEE system, expectations about systems under growth and innovation
can never be rational in the manner suggested by Lucas, a direct result
of Theorem \ref{Frame-relativity-A}, Corollary \ref{Frame-relativity-b}
and Lemma \ref{dense-undecidable-disjunctions}. Neither does the
concept of frame relativity depend on computability, the size of the
data set used to conduct inference, or the precision and accuracy
of inference devices: faster computers and better AI will not eradicate
expert error and failure. Theorizing about innovation within the neoclassical
paradigm is impossible due to its closed-endedness (Giménez Roche
2016), but we can theorize about innovation in an OEE knowledge framework
wherein individuals can be genuinely surprised and where nonlogical
schemes of reasoning like choosing based on one's ``gut feeling''
can be knowledge-generative. 

Given the necessity of nonlogical modes of reasoning in OEE systems,
we should expect increasing cultural, ideological and aesthetic fragmentation
in large and fast-moving societies, as it becomes more expensive to
test heuristic and patterned choice given the ramifying growth in
the adjacent possible. Social-institutional fragmentation is not necessarily
discoordinative. On the contrary, fragmentation in nonlogical knowledge-supportive
institutions represents the attempt of individuals to \emph{preserve}
their ability to coordinate with others under the pressures of rapid
system evolution. At the system-level, increased fragmentation also
means increased experimentation, competition and robustness of the
overall system to undesirable runaway phase transitions (spirals,
cycles, runs and busts as they are called in the economics literature).
If one type of nonlogical organization is more discoordinative than
coordinative, it can't spread too far. Particularly beneficial attempts,
on the other hand, can be observed and emulated. 

The weakness of traditional epistemology lies in the closed-ended
knowledge framework in which it is embedded. Presuming consistent,
closed partitions of a consistent, closed, and complete state space
is an exercise in point set topology, not an explication of human
knowledge acquisition. Probabilistic extensions of formal epistemology
as in Aumann (1999b) do no better, as they require the listability
of all possible statements. 

The epistemic logic of our OEE knowledge framework narrows the perspective
of the individual to its local context and its beliefs about reality,
which are in general different from the local contexts of other individuals
and their beliefs, and broadens the perspective of the individual
to the institutions to which it subscribes and in which it is embedded.
It also has a place for creativity and nonlogical schemes of thought
which have no place in traditional epistemology and rational choice
models. as individuals are aware of theory fragmentation across people
and through time and are aware that change in their own perspective
and change happening around them may happen logically or through ``leaps
of faith.'' 

The scientific success of Darwin was no less spectacular and complete
than that of Newton before him. And yet formal epistemology has largely
eschewed consideration of open-ended evolution. Taking the leap lands
us in a world of change and novelty in which state spaces are idiosyncratic
and different individuals know different, even contradictory, things.
It carries us from a world in which each mind is a copy of every other
mind and into a world in which there are many minds, each unique and
individual. In this world of many minds, rationality is less powerful,
and learning is non-algorithmic. Here, the growth of knowledge depends
not only on correct deduction and logical precision, but also on beauty,
joy, anger, hope, fear, poetry, zeal, and a vast host of other human
emotions, desires, and sensibilities. It is a disorienting world at
first, but it is a richer and more adventurous and ultimately more
rewarding world. Open-ended evolution creates a world of many minds
that should be explored by many minds. Who knows what is to be found
there?

\section{Appendix}

In this Appendix we provide the formal theory that underlies our framework. 

\subsection{Basic setup}

Consider a system with a set $N$ individuals. The universe $\Omega$
is defined as a collection of states $\omega$ that define the universe
in which individuals and their systems are embedded. Individuals develop
models $\mathbf{M}$ to represent their system based on theories $T$
in some formal language $L$, where $L$ has enough algebra to adequately
express physical phenomena---i.e., it interprets Peano arithmetic,
as do all the major theories of physical and social systems (Tsuji
et al 1998; Velupillai 2005; Chaitin 2017a). 

The language $L$ is a collection of variables and logical and nonlogical
constants. Logical constants are the usual connectives like $\wedge,\lor,\lnot$
and quantifiers like $\forall,\exists$. Nonlogical constants are
first-order formulas that signify some essential relationships (``predicates'')
and theorems. Call $\mathcal{P}_{i,t}$ the set of predicates that
generates $\Omega_{i,t}$ according to theory $T_{i,t}$.Aumann (1999a)
calls predicates ``tautologies''. Sentences $\xi$, also called
formulas or statements, are constructed as combinations of variables
and logical and nonlogical constants. The set of all sentences is
called the individual's syntax $\mathfrak{S}$. 

States $\omega\in\Omega$ are defined as the sentences in reference
to the individual's syntax $\mathfrak{S}$---that is, lists of formulas,
tautologies, predicates---that are true at that state, with respect
to a particular theoretical representation $T$ of the universe $\Omega$.
Epistemological environments depend on both an individual $i$ and
the time $t$ during which they acquire knowledge.\textbf{ }The ``frame''
of an individual $i$ is their known-world $\Omega_{i,t}$ and their
model of the known-world $\text{\textbf{M}}_{i,t}$ defined within
a theory $T_{i,t}$. In OEE systems, syntaxes $\mathfrak{S}_{i,t}$
are localized to the individual and do not apply to the entire population
$N$. 

Time ``ticks'' $t$ are not of a specific length and rather indicate
a change in an individual's known-world $\Omega_{i,t}\rightarrow\Omega_{i,t+1}$,
theory of the world $T_{i,t}\rightarrow T_{i,t+1}$, and associated
models and syntaxes. How individuals update these entities in OEE
systems is addressed in the ensuing analysis.

\subsection{Basic notation and definitions}

First, some notation from propositional logic. $\vdash$ means ``proves''
or ``is derivable from'', as in $T\vdash\xi$, ``the theory $T$
proves the sentence $\xi$'' or ``the sentence $\xi$ is derivable
from the theory $T$.'' We use the double right arrow $\implies$
for ``logically implies.'' The shorthand ``iff'' means ``if and
only if,'' or double-sided logical implication. We typically use
$\rightarrow$ to imply dynamical updating. The set-symbol $\setminus$
is the quotient, such that $\mathcal{P}_{i,t+1}\setminus\mathcal{P}_{i,t}\neq\emptyset$
means that subtracting the set $\mathcal{P}_{i,t}$ from the set $\mathcal{P}_{i,t+1}$
yields a nonempty ``remainder'' set. 

We utilize process notation, where in general $\mathbf{\overrightarrow{\mathbf{X}}}_{i,t}=\{X_{i,t},X_{i,t+1},X_{i,t+2},...\}$
is an open-ended process for individual $i$ that begins at time $t$,
$\mathbf{\overleftarrow{\mathbf{X}}}_{i,t}=\{...,X_{i,t-2},X_{i,t-1},X_{i,t}\}$
is the historical trajectory of the process for individual $i$ going
backwards from time $t$, and $\mathbf{X}_{i,t}=\{...,X_{i,t-2},X_{i,t-1},X_{i,t},X_{i,t+1},X_{i,t+2,}...\}$
is the process in its entirety.

We take the usual definitions of logical derivability and the logical
validity of sentences, in that a sentence $\xi$ is logically derivable
from a theory $T$ expressed in the language $L$ if $\xi$ is obtainable
from combining the sentences in $T$ with its logical axioms and the
operations of inference provided by the logical grammar of $L$, and
where a sentence $\xi$ is logically valid in a theory $T$ if it
can be derived from the logical axioms of $T$. 

We start with a few basic definitions necessary to prove Result 1
in particular.
\begin{defn}
\label{def-consistent}A theory $T$ is \textbf{consistent} if there
does not exist a $\xi\in T$ such that both $T\nvdash\xi$ and $T\vdash\xi$. 
\end{defn}
\begin{defn}
\label{def-coherent}A theory $T$ is \textbf{coherent} if $\lnot\xi\in T\implies\xi\notin T$.
\end{defn}
\begin{defn}
\label{def-complete}A theory $T$ is \textbf{complete} if $\forall\xi\in T$,
either $T\nvdash\xi$ or $T\vdash\xi$. Generally, a complete theory
defined this way must also be consistent.
\end{defn}
\begin{defn}
\label{def-axiomatizable}A theory $T$ is \textbf{axiomatizable}
if every valid sentence $\xi\in T$ can be derived from a recursive
set $R$ of sentences in $T$. Another way of putting this, useful
for our purposes, is that a theory $T$ is axiomatizable iff for some
decidable set of sentences $\Sigma$, the theory is the deductive
closure of that set, i.e., $T=Cn(\Sigma)$.
\end{defn}

\subsection{Axiomatizations}

Artemov (2022) characterizes the multi-agent model logic $S5_{n}$
(SEL) as follows:
\begin{itemize}
\item classical logic postulates and rule Modus Ponens $A,A\rightarrow B\vdash B$
\item distributivity: $K_{i}(A\rightarrow B)\rightarrow(K_{i}A\rightarrow K_{i}B)$
\item reflection: $K_{i}A\rightarrow A$
\item positive introspection: $K_{i}A\rightarrow K_{i}K_{i}A$
\item negative introspection:$\neg K_{i}A\rightarrow K_{i}\neg K_{i}A$
\item necessitation rule: $\vdash A\Rightarrow\vdash K_{i}A$.
\end{itemize}
In SEL terms, the statement-relevant axioms of traditional epistemology
can be summarized (with detailed definitions to follow) as:
\begin{enumerate}
\item Consistency (cf. Definition \ref{def-consistent})
\item Coherency (cf. Definition \ref{def-coherent})
\item Completeness (cf. Definition \ref{def-complete})
\item Epistemic closure of tautologies, or necessitation (cf. Definition
\ref{epistemic-closure})
\item Individual $i$ knows the knowledge partition of individual $j$,
and $j$ knows that $i$ knows, etc.
\end{enumerate}
These numbered axioms are explicitly derived from the modal logic
axioms of the system of modal logic $SE_{5}$ listed above in that
the SEL model of traditional epistemology is dual with $S5_{n}$ (cf.
Artemov 2022: 48). In the SEL interpretation of CE knowledge, suppose
there is a theory of everything $T_{\Omega}$ based on a language
$L_{\Omega}$ from which all correct statements about the world can
be derived, for all time. To reason coherently, individual $i$ must
reason from a subtheory $T_{i,t}$ of $T_{\Omega}$, or $T_{i,t}$
must be interpretable in $T_{\Omega}$ (these terms will be defined
below). 

\subsection{Formalizing open-endedness}

The model $\text{\textbf{M}}_{i,t}$ formalizes the knowledge of the
possible held by individual $i$ such that all valid sentences of
$T_{i,t}$ are realized in $\text{\textbf{M}}_{i,t}$: i.e., if $P_{1}\in\mathcal{P}_{i,t}$
is true in $T_{i,t}$, then it is true in $\text{\textbf{M}}_{i,t}$. 
\begin{defn}
\label{Define-the-possible}Define the \textbf{possible} for individual
$i$ at time $t$ as a triplet $\mathbf{\varPi}_{i,t}=\langle\text{\textbf{M}}_{i,t},T_{i,t},\mathcal{P}_{i,t}\rangle$
.
\end{defn}
\begin{defn}
\label{Define-the-adjacent}Define the \textbf{adjacent possible}
for individual $i$ at time $t$ as a triplet $\varPi_{i,t+1}=\langle\text{\textbf{M}}_{i,t+1},T_{i,t+1},\mathcal{P}_{i,t+1}\rangle$
.
\end{defn}
Open-endedness implies an asymmetry in consistency and completeness
of individual theories of the universe. We describe each of these
as a case requiring an individual's theory to be revised.
\begin{casenv}
\item Suppose individual $i$ has a consistent, coherent and complete theory
$T_{i,t}$ of $\Omega_{i,t}$ that does not remain consistent for
$\Omega_{i,t+1}$. Then,  $i$ must revise or replace their theory
$T_{i,t}\rightarrow T_{i,t+1}$ so that $T_{i,t+1}$ is consistent,
coherent and complete for all sentences in $\Omega_{i,t+1}$. 
\item Suppose individual $i$ has a consistent, coherent and complete theory
$T_{i,t}$ of $\Omega_{i,t}$ that remains consistent for $\Omega_{i,t+1}$
but upon the observation of new, unaccounted-for variables is no longer
complete. That is, the localized syntax $\mathfrak{S}_{i,t+1}\setminus\mathfrak{S}_{i,t}\neq\emptyset$.
Even though the underlying theory of how variables of certain types
relate to each other still follows, the entrepreneur will still have
to update their syntax $\mathfrak{S}_{i,t}\rightarrow\mathfrak{S}_{i,t+1}$
and thus their theory $T_{i,t}\rightarrow T_{i,t+1}$ such that $T_{i,t+1}$
is consistent, coherent and complete for all sentences in $\Omega_{i,t+1}$. 
\end{casenv}
In order to define what we mean by open-endedness, we construct open-ended
movement through the possible for individual $i$ as the process
\begin{lyxcode}
\begin{equation}
\mathbf{\overrightarrow{\Pi}_{i,t}=}\{\langle\text{\textbf{M}}_{i,t},T_{i,t},\mathcal{P}_{i,t}\rangle,\langle\text{\textbf{M}}_{i,t+1},T_{i,t+1},\mathcal{P}_{i,t+1}\rangle,\langle\text{\textbf{M}}_{i,t+2},T_{i,t+2},\mathcal{P}_{i,t+2}\rangle,...\}\label{eq:open-ended-process}
\end{equation}
\end{lyxcode}
\begin{defn}
\label{Define-open-ended}The process $\mathbf{\Pi}_{i}$ is \textbf{open-ended}
if $\mathcal{P}_{i,t+1}\setminus\mathcal{P}_{i,t}\neq\emptyset$.
i.e., new predicates are added to an individual's known-world at each
time step.
\end{defn}
Can observers in OEE systems infer a path of theoretical revision
$\{T_{i,0}\rightarrow T_{i,1}\rightarrow...\rightarrow T_{i,t}\rightarrow...\rightarrow T_{\Omega}\}$
that progresses towards the theory-of-everything $T_{\Omega}$? 
\begin{defn}
\label{def-subtheory}A theory $T_{1}$ is a \textbf{subtheory} of
a theory $T_{2}$ if all valid sentences in $T_{1}$ are also valid
in $T_{2}$. The theory $T_{2}$ is called an \textbf{extension} of
$T_{1}$ .
\end{defn}
\begin{defn}
\label{Inessential-extension}A theory $T_{2}$ is called an \textbf{inessential
extension} of $T_{1}$ if its sentences are derivable from the sentences
of $T_{1}$. 
\end{defn}
\begin{thm}
\label{Tplus-essential-extension}If $T_{i,t+1}$ is an extension
of \textup{$T_{i,t}$, where }$T_{i,t+1}$ \textup{and }$T_{i,t}$
are theories of open-ended systems as defined in Definition \ref{Define-the-adjacent}\textup{
}then $T_{i,t+1}$ is in general an essential extension of $T_{i,t}$.
\end{thm}
\begin{proof}
Suppose $T_{i,t+1}$ is, in general, an inessential extension of $T_{i,t}$.
Then by Definition \ref{Inessential-extension}, $T_{i,t+1}$ must
share all the nonlogical constants of $T_{i,t}$. By Definition \ref{Define-open-ended},
$\mathcal{P}_{i,t+1}\setminus\mathcal{P}_{i,t}\neq\emptyset$. So
$T_{i,t+1}$ cannot be an inessential extension of $T_{i,t}$.
\end{proof}
\begin{cor}
Observers in OEE systems cannot in general infer a path of theoretical
revision that converges to $T_{\Omega}$. 
\end{cor}
\begin{proof}
Any $T_{i,t+1}$ in OEE systems is not in general an inessential extension
of $T_{i,t}$ by Theorem \ref{Tplus-essential-extension}. By Definition
\ref{Inessential-extension}, $T_{i,t+1}$ is not in general derivable
from $T_{i,t}$. Therefore, observers cannot in general infer a path
$\{T_{i,0}\rightarrow T_{i,1}\rightarrow...\rightarrow T_{i,t}\rightarrow...\rightarrow T_{\Omega}\}$
that progresses towards the theory-of-everything $T_{\Omega}$.
\end{proof}

\subsection{The canonical OEE formalism}

Recall that traditional epistemology glosses the origin of an individual
$i$'s knowledge partition $\mathscr{I}$, which is constructed using
$\kappa_{i}$ such that $\mathscr{I}(\omega)=\{\omega'\in\Omega:\kappa_{i}(\omega)=\kappa_{i}(\omega')\}$,
where $\Omega$ is the universe of all states $\omega\in\Omega$. 

Suppose we have a logically closed list $\mathfrak{L}$ of sentences
$\xi$ in a language $L$. In the syntactic formalism for the embedded
observer $i$, there are two kinds of sentences $\xi$ in $\mathfrak{L}$:
sentences $f$ prepended with $k_{i}$ such that $k_{i}f$ means ``$i$
knows $f$'', and sentences $g$ not prepended with $k_{i}$. 
\begin{defn}
\label{epistemic-closure}A list \textbf{$\mathfrak{L}$ }is \textbf{epistemically
closed} if it satisfies the condition that $f\in\mathfrak{L}\Rightarrow k_{i}f\in\mathfrak{L}$. 
\end{defn}
Epistemic closure is an axiom of traditional epistemology (necessitation).
Epistemic closure is a non-obvious epistemological assumption, as
how individual $i$ comes to prepend a sentence $f$ with $k_{i}$---how
$i$ comes to know a statement is true with respect to a particular
state---is the question epistemology is intended to answer. Traditional
epistemology does not assume that lists are epistemically closed but
it does assume the set of non-state-specific sentences (tautologies)
are epistemically closed (\emph{necessitation} in SEL) and that $k_{i}f\Rightarrow f$
is a tautology. 

In our OEE knowledge framework we can neither assume that the set
of non-state-specific sentences is epistemically closed (\emph{necessitation})
\emph{nor} can we assume that $k_{i}f\Rightarrow f$ is a tautology
(\emph{reflection}). OEE systems break the rule of necessitation at
the observer-level, requiring us to define an individual's \emph{local
knowledge} as a distinct concept. We shall proceed by first demonstrating
a few properties of the individual's localized sets. 
\begin{prop}
Suppose $\mathfrak{T}_{i,t}$ is the list of all tautologies defined
within $T_{i,t}$. Then, (i) $\mathfrak{T}_{i,t}$ is epistemically
closed, and (ii) $\mathfrak{T}_{i,t}=\mathcal{P}_{i,t}$.
\end{prop}
\begin{proof}
For part (i), note that individual $i$ constructs $T_{i,t}$ at each
step by Definition \ref{Define-the-possible}. This implies that all
the logical and nonlogical axioms of $T_{i,t}$ (which include the
logical and nonlogical axioms of ZFC) are known to $i$, and thus
are in $\mathfrak{T}_{i,t}$. Individuals also assume that their logical
systems are consistent. This implies that any consistent deduction
of $T_{i,t}$, using standard logic is included in $\mathfrak{T}_{i,t}$.
For part (ii), suppose $\exists\xi\in\mathfrak{T}_{i,t},\xi\notin\mathcal{P}_{i,t}$.
Then, either $T_{i,t}$ is incomplete or inconsistent, which contradicts
the assumption of rationality.
\end{proof}
\begin{prop}
Suppose $\mathfrak{T}$ is the list of all tautologies of $T_{\Omega}$.
Then (i) $\mathfrak{T}$ is not epistemically closed with respect
to any individual $i$, and (ii) $\mathfrak{T}_{i,t}\varsubsetneq\mathfrak{T}$
for any $i\in N,t\in\mathbb{N}$.
\end{prop}
\begin{proof}
The proof is a simple implication of Definitions \ref{Define-the-possible}
and \ref{Define-open-ended}. If $\mathfrak{T}$ is epistemically
closed, that implies that all nonlogical constants in the universe
are known to $i$. But that implies $\mathcal{P}_{\Omega}\setminus\mathcal{P}_{i,t}=\emptyset$,
which contradicts Definition \ref{Define-open-ended}. For (ii), if
$\mathfrak{T}_{i,t}=\mathfrak{T}$ then $\mathcal{P}_{\Omega}=\mathcal{P}_{i,t}$
and contradicts Definition \ref{Define-open-ended}. If $\mathfrak{T}_{i,t}\supset\mathfrak{T}$,
then $\mathcal{P}_{i,t+1}\setminus\mathcal{P}_{i,t}=\emptyset$. But,
again this contradicts Definition \ref{Define-open-ended}.
\end{proof}
Non-state-specific predicates of a system constitute the theorems
of a consistent, complete and decidable theory $T_{\Omega}$ based
on language $L$. $T_{\Omega}$ generates the universe $\Omega$.
In traditional epistemology, the theory of the universe $T_{\Omega}$
and the theory accessible by individuals $i$ and $j$ to comprehend
their universe are the same: $T_{\Omega}=T_{i,t}=T_{j,t}$ ( see Lemma
8.71 of Aumann (1999a)). But OEE systems have not yet innovated their
final state, and so individuals cannot in general access all salient
statements in the system by Definition \ref{Define-open-ended} and
Theorem \ref{Tplus-essential-extension}. 

In traditional epistemology, an individual $i$ is granted knowledge
of some subset of the true values of state-specific statements, exactly
those statements prepended with $k_{i}$, and under epistemic closure
the set of all ``tautologies,'' or the simplest set of simple non-state-specific
predicates that completely generate the theory $T$ of the universe
$\Omega$. Call the list of known tautologies and state-specific statements
$\mathfrak{K}_{i}$. 

Suppose there is a space $\Sigma_{it}=\{0,1\}^{\Omega_{i,t}}$ that
lists 0 or 1 with respect to each unique sentence $\xi$ for each
state $\omega\in\Omega_{i,t}$. Then, $i$ employs a decision procedure
$\delta_{it}:\Omega_{i,t}\rightarrow\Sigma_{i,t}$ that maps sentences
in states to their truth values in $T_{i,t}$, where
\begin{lyxcode}
\begin{equation}
\delta_{it}:\Omega_{i,t}\rightarrow\Sigma_{i,t},\delta_{i,t}(\xi)+\delta_{i,t}(\lnot\xi)=1,\forall\xi\in\mathcal{P}_{i,t}\label{eq:def-decision-function}
\end{equation}
\end{lyxcode}
$\delta_{it}$ is a simple function with a constraint\footnote{The astute reader might realize that $\Sigma_{it}$ is essentially
Borel's number---or equivalently, Chaitin's $\Omega$---for the
individual's known-universe $\Omega_{i,t}$ (Chaitin 2005). $\delta_{it}$,
then, queries the Delphic oracle.}. We first note that
\begin{prop}
\label{prop:delta-defined-all-sentences}A decision procedure $\delta_{i,t}$
defined for a complete and consistent theory $T_{i,t}$ must be defined
over all $\xi\in\Omega_{i,t}$.
\end{prop}
\begin{proof}
This is easy to see from the formulation of $\delta_{i,t}$. Simply,
if $\delta_{it}$ cannot decide some sentence $\xi$, then $T_{i,t}$
is not provably consistent by Definition \ref{def-consistent}. $T_{i,t}$
is also not provably complete, as we cannot prove whether $T_{i,t}$
entails $\xi$ or $\lnot\xi$ , by Definition \ref{def-complete}.
\end{proof}
Construct the set $\mathfrak{K}_{i,t}$ as 
\begin{lyxcode}
\begin{equation}
\mathfrak{K}_{i,t}:=\{\xi\in\Omega_{i,t}:\delta_{i,t}(\xi)+\delta_{it,}(\lnot\xi)=1\}\label{eq:def-knowledge-list}
\end{equation}
\end{lyxcode}
\begin{defn}
\label{def-contextual-knowledge}The \textbf{contextual knowledge
possible} for individual $i$ at time $t$ as is the set $\mathcal{K}_{i,t}$
where
\end{defn}
\begin{lyxcode}
\begin{equation}
\mathcal{K}_{i,t}:=\{\omega\in\Omega_{i,t}:\delta_{i,t}(\omega)\in\Sigma_{it}\}\label{eq:def-knowledge-possible}
\end{equation}
\end{lyxcode}
\begin{defn}
\label{def-local-knowledge}The \textbf{local knowledge} of individual
$i$ is $\kappa_{i,t}(\omega)$, the set of all sentences $\xi\in\omega$
that start with $k_{i,t}$. We can also construct the same set by
prepending all sentences $\xi\in\mathfrak{K}_{i,t}$ with $k_{i,t}$.
\end{defn}
Defining $\kappa_{i,t}$ allows us to relate an individual's knowledge
directly to the system state-space $\Omega$. While $\kappa_{i,t}$
is defined on $\Omega$, it is constrained by being constructed from
$\mathfrak{K}_{i,t}$ and, ultimately, $T_{i,t}$. To understand how
knowledge is updated in an OEE system, we take our definition of the
``knowledge possible'' to construct a definition of knowledge-centered
adjacent possible:
\begin{defn}
Define the \textbf{adjacent knowledge possible} for individual $i$
at time $t$ as the quotient of the temporally adjacent contextual
knowledge possible $\mathcal{K}_{i,t+1}$ with the current contextual
knowledge possible $\mathcal{K}_{i,t+1}$:
\end{defn}
\begin{lyxcode}
\begin{equation}
\mathcal{A}_{i,t}\equiv\mathcal{K}_{i,t+1}\setminus\mathcal{K}_{i,t}=\{\omega':\omega'\in\mathcal{K}_{i,t+1},\omega'\notin\mathcal{K}_{i,t}\}.\label{eq:adjacent-states}
\end{equation}
\end{lyxcode}
\begin{prop}
\label{adjacent-nonempty}In open-ended evolutionary systems, the
adjacent knowledge possible $\mathcal{A}_{i,t}$ is non-empty.
\end{prop}
\begin{proof}
Suppose $\mathcal{A}_{i,t}=\emptyset$. Then $\mathcal{P}_{i,t+1}=\mathcal{P}_{i,t}$
by the definition of $\delta_{i,t}$. But OEE systems are defined
such that $\mathcal{P}_{i,t+1}\setminus\mathcal{P}_{i,t}\neq\emptyset$.
Therefore, $\mathcal{A}_{i,t}\neq\emptyset$.
\end{proof}
\begin{lem}
\label{undecidable-sentence-exists}There exists a sentence $\xi'\in\omega'$
for $\omega'\in\Omega$ such that $\delta_{i,t}(\xi')+\delta_{it,}(\lnot\xi')\neq1$.
\end{lem}
\begin{proof}
This is a direct consequence of the nonemptiness of Equation \ref{eq:adjacent-states}. 
\end{proof}
As we know from the axioms of traditional epistemology in propositional
logic summarized above, key part of obtaining predictive dynamics
for the traditional canonical syntactic epistemology (Aumann 1999a:
276; Samet 1990) is establishing that states of the world $\omega\in\Omega$
as seen by the individual are closed, coherent (see Definition \ref{def-coherent}),
and complete (see Definition \ref{def-complete}). 

As we may expect, states $\omega\in\Omega_{i,t}$ for OEE systems
do not, in general, exhibit such analytically nice properties. 
\begin{thm}
\label{coherence-local-knowledge}States $\omega\in\mathcal{K}_{i,t}$
in an open-ended evolutionary system are coherent, but are neither
complete nor contain all tautologies. 
\end{thm}
\begin{proof}
Suppose both $\lnot\xi\in\omega$ and $\xi\in\omega$ for some $\omega\in\mathcal{K}_{i,t}$.
Then $\delta_{i,t}(\xi)+\delta_{it,}(\lnot\xi)\neq1$ for this $\omega\in\mathcal{K}_{i,t}.$
But this would mean that $\delta_{i,t}(\omega)\notin\Sigma_{it}$,
which contradicts Equation \ref{eq:def-knowledge-possible}. So $(\lnot\xi\in\omega\land\xi\in\omega)\implies\omega\notin\mathcal{K}_{i,t}$,
and thus $\lnot\xi\in\omega\implies\xi\notin\omega$. This establishes
coherence, by Definition \ref{def-coherent}. Suppose $\xi'\notin\omega$
for some $\omega\in\mathcal{K}_{i,t}$, and suppose this implies $\lnot\xi'\in\omega$.
But then this would imply that $\delta_{i,t}(\xi)+\delta_{it,}(\lnot\xi)=1$
for all $\omega\in\Omega$, which contradicts Lemma \ref{undecidable-sentence-exists}.
So, $\mathcal{K}_{i,t}$ is incomplete. For the last part of the proof,
$\mathcal{K}_{i,t}$ cannot contain all tautologies of the universe
$\Omega$ as a direct result of Proposition \ref{adjacent-nonempty},
and by the definition of OEE systems. 
\end{proof}
In OEE systems, pairs of hierarchies $(h_{i,t},h_{j,t})$ are, in
general, not mutually consistent, which prevents the construction
of general common knowledge in OEE systems.
\begin{prop}
\label{hierarchies-not-mutually-consistent}In open-ended evolutionary
systems, pairs of hierarchies $(h_{i,t},h_{j,t})$ are, in general,
not mutually consistent.
\end{prop}
\begin{proof}
Suppose they are consistent. Since $T_{i,t}\ne T_{\Omega},T_{j,t}\ne T_{\Omega}$,
this condition holds when $T_{i,t}=T_{j,t},T_{i,t}\subset T_{\Omega},T_{i,t}\subset T_{\Omega}$.
Suppose that, in general, $T_{i,t}=T_{j,t}$. But this would imply
that $\mathcal{K}_{i,t}=\mathcal{K}_{j,t}$ and $\Omega_{i,t}=\Omega_{j,t}$
in general, which is not true under open-ended evolution. 
\end{proof}
\begin{cor}
\label{Common-knowledge-impossible}Common knowledge is generally
impossible under open-ended evolution.
\end{cor}
\begin{proof}
Suppose individuals $i$ and $j$ believe they inhabit $\Omega_{i,t}$
and $\Omega_{j,t}$, respectively. Each $T_{i,t}$ is a complete and
consistent theory of each $\Omega_{i,t}$, and each $\mathcal{P}_{i,t}$
is the generative list of nonlogical constants, i.e., predicates.
Under open-ended evolution, $T_{i,t}\ne T_{j,t}\ne T_{\Omega}$ in
general, which implies that $\mathcal{P}_{i,t}\ne\mathcal{P}_{j,t}\ne\mathcal{P}_{\Omega}$.
That is, $\exists\xi\in T_{i,t}$ and $\exists\xi'\in T_{j,t}$ such
that neither $\xi,\lnot\xi\in T_{j,t}$ nor $\xi',\lnot\xi'\in T_{i,t}$
. Suppose $i,j$ experience some ``true'' state $\omega\in\Omega$
that $i$ interprets as some $\omega_{i,t}\in\Omega_{i,t}$ and $j$
as some $\omega_{j,t}\in\Omega_{j,t}$. Suppose $k_{i,t}\xi\in\omega_{i,t}$
where $\xi,\lnot\xi\notin T_{j,t}$. Then $k_{j,t}\xi,k_{j,t}\lnot\xi\notin\omega_{j,t}$.
But this means that $\xi\notin\mathfrak{K}_{j,t}$ and thus, by definition,
no states $\omega_{j,t}$ where $\xi\in\omega_{j,t}$ can be in $\mathcal{K}_{j,t}$.
So, the states individuals consider possible are not entirely coincident,
which means that their local knowledge is in general non-coincident
$\mathcal{K}_{i,t}\ne\mathcal{K}_{j,t}$. Therefore there is no natural
presumption of common knowledge in open-ended evolution.
\end{proof}
\begin{cor}
\label{agree-to-disagree} It is possible for individuals to ``agree
to disagree'' under open-ended evolution.
\end{cor}
\begin{proof}
This is a direct consequence of Proposition \ref{hierarchies-not-mutually-consistent}
and Corollary \ref{Common-knowledge-impossible}.
\end{proof}
The canonical form of the epistemic logic of OEE systems is axiomatically
weaker than traditional epistemology by at least two of the five axioms
of $SE_{5}$ (\emph{necessitation} and \emph{reflection}). We will
show below that we can further weaken the axiomatic infrastructure
of the epistemic logic of OEE systems.

\subsection{Disjoint knowledge under open-ended evolution}

Observer $i$ can distinguish between two states $\omega,\omega'\in\Omega_{i,t}$
iff $\kappa_{i,t}(\omega)\ne\kappa_{i,t}(\omega')$. If $\xi\in\omega$
= ``it is cloudy'' for some $\omega\in\Omega_{i,t}$ and $k_{i,t}\xi\in\kappa_{i,t}(\omega)$
but $k_{i,t}\xi\notin\kappa_{i,t}(\omega')$ for some other state
$\omega'\in\Omega_{i,t}$, then either it is not cloudy (i.e., $k_{i,t}(\lnot\xi)\in\kappa_{i,t}(\omega')$)
or the individual does not know whether it is cloudy (i.e., $k_{i,t}\xi,k_{i,t}(\lnot\xi)\notin\kappa_{i,t}(\omega')$). 

Recall in the following theorem due to Tarski (2010 {[}1953{]}: 14-15): 

\begin{defn}
\label{def-decidable}A theory $T$ is \textbf{decidable} if the set
of all its valid sentences is recursively enumerable. Otherwise, $T$
is \textbf{undecidable}.
\end{defn}
\begin{thm}
\label{Decidable-essentially-not-axiomatizable}For a complete theory
$T$ the following three conditions are equivalent: (i) $T$ is undecidable,
(ii) $T$ is essentially undecidable, and $(iii)$ $T$ is not axiomatizable.
\end{thm}
\begin{proof}
As noted in (Tarski et al. 1953 {[}2010{]}: 14-15), the proof of how
(i) implies (iii) is a consequence of Gödel (1931: 56, Theorem V),
and the rest of the proof follows by definition.
\end{proof}
\begin{thm}
\label{oeee-is-undecidable}The epistemic logic of OEE systems is
undecidable by any observer $i$ embedded in a system $\Omega$.
\end{thm}
In order to prove Theorem \ref{oeee-is-undecidable}, we need a few
other items, namely, the results of Gödel and some basic implications.
\begin{thm}
\label{godel-ICT}Suppose Peano arithmetic (PA) is interpretable in
some theory $T$ in a language $L$. Then there does not exist a decision
procedure $\delta_{i,t}$defined on $T$ such that $\delta_{i,t}(\xi)+\delta_{it,}(\lnot\xi)=1$
for all sentences $\xi$ in $L$. 
\end{thm}
\begin{proof}
This is the famous proof due to Gödel (1931).
\end{proof}
\begin{cor}
Suppose there exists a population of $N$ individuals in an open-ended
evolutionary system. Then for all $T$ that contain enough arithmetic,
there does not exist a $T_{i,\tau}$such that $T_{i,\tau}=T_{\Omega}$. 
\end{cor}

\subsection{Frame relativity results}

Our first implication of the canonical OEE epistemological model is
that decision procedures $\delta_{i,t}$ are in general incomplete
with respect to the ontological truth of the universe, the theory-of-everything
$T_{\Omega}$. We call this result ``frame relativity,'' since it
implies that individual $i$'s frame of understanding at time $t$---their
possible $\mathbf{\varPi}_{i,t}=\langle\text{\textbf{M}}_{i,t},T_{i,t},\mathcal{P}_{i,t}\rangle$---is
an incomplete picture of the actual universe $\Omega$. Frame relativity
in OEE systems is proved in two parts, below.
\begin{prop}
\label{Frame-relativity-A}\emph{(}\textbf{\emph{Frame Relativity
A}}\emph{)} Any decision procedure $\delta_{i,t}$ consistent with
individual $i$'s current model of the universe $\mathbf{M}_{i,t}$
based on a theory $T_{i,t}$ is incomplete with respect to the theory-of-everything
$T_{\Omega}$. That is, under open-ended evolution, no individual
embedded in the universe can have a complete model of the universe.
\end{prop}
\begin{proof}
Suppose there exists a $\delta_{i,t}$ that completes $T_{i,t}$ such
that $T_{i,t}=T_{\Omega}$. Since $i$ is in an OEE system, there
exists a $\xi'$ such that $\delta_{i,t}(\xi')+\delta_{it,}(\lnot\xi')\neq1$,
by Lemma \ref{undecidable-sentence-exists}. But then $T_{i,t}$ is
not complete, by Definition \ref{def-complete}.
\end{proof}
\begin{prop}
\label{m-does-not-halt}Suppose individual $i$ in system $\Omega$
possesses an algorithmic process $m$ that decides undecidable disjunctions
for an OEE process $\overrightarrow{\mathbf{\Pi}}_{i,t}$ whenever
disjunctions are encountered, resulting in subsequent extensions of
some base theory $T_{i,t}$. Then there exists a sentence $\xi'\in\Omega$
for which $m$ is not specified.
\end{prop}
\begin{proof}
Suppose we start with the theory $T_{i,t}$ of the system $\Omega_{i,t}$.
By Theorems \ref{Frame-relativity-A} and \ref{godel-ICT}, there
exists an $\xi\in\Omega_{i,t+1}$ such that $T_{i,t}\nvdash(\xi\vee\lnot\xi)$.
Assume that the algorithmic process iteratively decides each undecidable
disjunction and extends theories iteratively. Suppose $m$ halts at
theory $T_{i,\tau}$. But then that implies $T_{i,\tau}$ is complete,
which runs counter to \ref{godel-ICT} and \ref{Frame-relativity-A}.
Therefore, $m$ does not halt, meaning that regardless how long $m$
runs, there will always exist a sentence $\xi'\in\Omega$ that $m$
cannot decide. 
\end{proof}
\begin{cor}
\label{Frame-relativity-b} \emph{(}\textbf{\emph{Frame Relativity
B}}\emph{)} Consider an OEE process $\mathbf{\overrightarrow{\mathbf{\Pi}}}_{i,t}$
for individual $i$ as described by Equation \ref{eq:open-ended-process}.
Suppose a theory of everything $T_{\Omega}$ is built through an algorithmic
process $m$ from a sequence of theories $\{T_{i,0},T_{i,1},...,T_{i,t},T_{i,t+1},...\}$,
and is therefore recursively enumerable. If $T_{i,t}\subset T_{i,t+1}$
where $T_{i,t}$ is a subtheory of $T_{i,t+1}$ and $T_{i,t+1}$ an
extension of $T_{i,t}$, then $T_{\Omega}$ is undecidable.
\end{cor}
\begin{proof}
By Proposition \ref{m-does-not-halt}, $T_{i,t}$ is not recursively
enumerable. Therefore, $T_{i,t}$ is not axiomatizable by Definition
\ref{def-axiomatizable}. By Theorem \ref{Decidable-essentially-not-axiomatizable},
$T_{i,t}$ is therefore not decidable. By iterative induction, $T_{\Omega}$
is therefore undecidable by $i$.
\end{proof}
There exists an infinite number of undecidable disjunctions $\mathfrak{D}$
which are decided in $T_{\Omega}$ (which have a definite truth value)
but which cannot be decided by $i$ reasoning with $T_{i,t}$ in any
given known-world $\Omega_{i,t}$. That is, a theory $T_{i,t}$ of
$\Omega_{i,t}$ induces an uncountably infinite set of undecidable
disjunctions $\mathfrak{D}_{i,t}$. 
\begin{lem}
\label{dense-undecidable-disjunctions}$T_{i,t}$ is dense in undecidable
disjunctions for any individual $i\in N$ and any time $t$. 
\end{lem}
\begin{proof}
Suppose there are a finite number of undecidable disjunctions $\mathfrak{D}_{i,t}$
in the theory $T_{i,t}$ of the individual's known-world $\Omega_{i,t}$
at time $t$. Suppose the length of the list $|\mathfrak{D}_{i,t}|=\tau$.
Then, by the definition of open-ended evolution, $T_{i,t+\tau}=T_{\Omega}$.
But this implies $T_{i,t+\tau}$ is complete and decidable,, and more
importantly, that $T_{\Omega}$ is decidable, which it cannot be due
to Theorem \ref{godel-ICT} and Frame Relativity B (Corollary \ref{Frame-relativity-b}).
Therefore, in the perspective of individual $i$ in known-world $\Omega_{i,t}$
for any time $t$, there exists an infinite number of undecidable
disjunctions $\mathfrak{D}_{i,t}$. 
\end{proof}

\end{document}